\documentclass[12pt,a4paper,oneside]{article}

\usepackage{natbib}
\setcitestyle{citesep={,}}
\usepackage[utf8]{inputenc} 
\usepackage[T1]{fontenc}    
\usepackage{hyperref}       
\usepackage{url,soul}       
\usepackage{amsfonts}       
\usepackage{nicefrac}       
\usepackage{microtype}      
\usepackage{graphicx}
\usepackage{xcolor}
\usepackage[a4paper,left=2cm,right=2cm,top=2.5cm,bottom=2.5cm]{geometry}
\usepackage{placeins} 

\usepackage[USenglish]{babel}
\usepackage{color,ulem}
\usepackage{amsmath} 
\usepackage{amsthm} 
\usepackage{amssymb} 
\usepackage{bm} 

\definecolor{light-gray}{gray}{0.95}

\usepackage{mdframed}
\mdfdefinestyle{MyFrame}{%
    linecolor=black,
    outerlinewidth=1pt,
    roundcorner=5pt,
    innertopmargin=\baselineskip,
    innerbottommargin=\baselineskip,
    innerrightmargin=30pt,
    innerleftmargin=30pt,
    backgroundcolor=light-gray,
    startinnercode={\baselineskip=0.6cm}}

\newtheorem{theorem}{Theorem}[section]

\newtheorem{lemma}[theorem]{Lemma}
\newtheorem{remark}{Remark}
\newtheorem{assumption}{Assumption}
\newtheorem{definition}{Definition}

\newcommand{\tk}{\theta_k}
\newcommand{\pk}{\psi_k}
\newcommand{\dt}{\dot{\theta}}
\newcommand{\ddp}{\dot{\psi}}
\newcommand{\ovb}{\frac{1}{\beta}}
\newcommand{\ma}{\mathsf{S}}
\newcommand{\ml}{\mathsf{L}}
\newcommand{\mb}{\mathsf{B}}
\newcommand{\N}{\mathbb{N}}
\usepackage{bbm}
\newcommand{\indic}{\mathbf{1}}
\newcommand{\J}{\mathcal{J}}
\newcommand{\R}{\mathbb{R}}

\newcommand*\diff{\mathop{}\!\mathrm{d}}
\newcommand*\dist{\mathop{}\!\mathrm{dist}}

\newcommand{\aaa}{(\alpha-\ovb)}
\newcommand{\bb}{\ovb}

\newcommand{\vv}{\frac{\dt(t)-\ddp(t)}{\beta}}

\newcommand{\new}[1]{{#1}} 

\DeclareRobustCommand{\VAN}[3]{#2} 

\title{An Inertial Newton Algorithm for Deep Learning}

\usepackage{lastpage}

\author{\textbf{Camille Castera}\footnote{Corresponding author: \texttt{camille.castera@protonmail.com}}\\
 IRIT, Universit\'e de Toulouse, CNRS\\
Toulouse, France
\and
\textbf{J\'er\^ome Bolte}$^\dagger$ \\
Toulouse School of Economics\\ Université de Toulouse\\
Toulouse, France
\and
\textbf{Cédric Févotte}$^\dagger$ \\
 IRIT, Universit\'e de Toulouse, CNRS\\
Toulouse, France
\and
\textbf{Edouard Pauwels}$^\dagger$ \\
IRIT, Universit\'e de Toulouse, CNRS\\
DEEL, IRT Saint Exupery\\
Toulouse, France
}

\begin{document}
\sloppy

	\maketitle
    
    \renewcommand*{\thefootnote}{$^\dagger$}
	\footnotetext[1]{Last three authors are listed in alphabetical order. }
    \renewcommand*{\thefootnote}{\arabic{footnote}}
    \setcounter{footnote}{0} 
    \begin{abstract}%
    We introduce a new second-order inertial optimization method for machine learning called INNA. It exploits the geometry of the loss function while only requiring stochastic approximations of the function values and the generalized gradients. This makes INNA fully implementable and adapted to large-scale optimization problems such as the training of deep neural networks. The algorithm combines both gradient-descent and Newton-like behaviors as well as inertia.
    We prove the convergence of INNA for most deep learning problems. To do so, we provide a well-suited framework to analyze deep learning loss functions involving tame optimization in which we study a continuous dynamical system together with its discrete stochastic approximations.  We prove sublinear convergence for the continuous-time differential inclusion which underlies our algorithm. Additionally, we also show how standard optimization mini-batch methods applied to non-smooth non-convex problems can yield a certain type of spurious stationary points never discussed before. We address this issue by providing a theoretical framework around the new idea of $D$-criticality; we then give a simple asymptotic analysis of INNA. Our algorithm allows for using an aggressive learning rate of $o(1/\log k)$. From an empirical viewpoint, we show that INNA returns competitive results with respect to state of the art (stochastic gradient descent, ADAGRAD, ADAM) on popular deep learning benchmark problems.
	\end{abstract}%


	\section{Introduction}

	Can we devise a learning algorithm for general deep neural networks (DNNs) featuring inertia and Newtonian directional intelligence only by means of backpropagation? In an optimization jargon: can we use second-order ideas in time and space for {\em non-smooth non-convex} optimization by only using a subgradient oracle?
    Before providing some answers to this  question, let us have a glimpse at some fundamental optimization algorithms for training DNNs.

	The backpropagation algorithm  is, to this day, the fundamental block to compute gradients in deep learning (DL). It is used in most instances of the Stochastic Gradient Descent (SGD) algorithm  \citep{robbins1951stochastic}. The latter is powerful, flexible, capable of handling large-size problems, noise, and further comes with theoretical guarantees of many kinds. We refer to \cite{bottou2008tradeoffs,moulines2011non} in a convex machine learning context and \cite{bottou2018optimization} for a recent account highlighting the importance of DL applications and their challenges. In the non-convex setting, recent works of \cite{adil,davis2018stochastic} follow the \textit{Ordinary Differential Equations (ODE) approach} introduced in \cite{ljung1977analysis}, and further developed in \cite{benaim1999dynamics,kushner2003stochastic,benaim2005stochastic,borkar2009stochastic}. 
	Two research directions have been explored in order to improve SGD's training efficiency:
    \begin{itemize}
    \renewcommand\labelitemi{--}
        \item using local geometry of empirical loss functions to obtain {steeper} descent directions,
        \item using past steps history to design {larger step-sizes} in the present.
    \end{itemize}
	The first approach is akin to quasi-Newton methods while the second revolves around Polyak's inertial method \citep{polyak1964some}.
	The latter is inspired by the following appealing mechanical thought-experiment. Consider a heavy ball evolving on the graph of the loss function (the loss function's \textit{landscape}), subject to gravity and stabilized by some friction effects. Friction generates energy dissipation, so that the particle will eventually reach a steady state which one hopes to be a local minimum. These two approaches are already present in the DL literature: among the most popular algorithms for training DNNs, ADAGRAD \citep{duchi2011adaptive} features  local geometrical aspects while ADAM \citep{kingma2014adam} combines inertial ideas with step-sizes similar to the ones of ADAGRAD. Stochastic Newton and quasi-Newton algorithms have been considered by \cite{martens2010deep,byrd2011use,byrd2016stochastic} and recently reported performing efficiently on several problems \citep{berahas2017investigation,xu2017second}. The work of \cite{wilson2017marginal} demonstrates that carefully tuned SGD and heavy-ball algorithms are competitive with concurrent methods.

	However, deviating from the simplicity of SGD also comes with major challenges because of {the high dimensionality of DL problems} and the severe absence of regularity in DL (differential regularity is generally absent, but even weaker regularity such as semi-convexity or Clarke regularity are not {always} available). All sorts of practical and theoretical hardships are met: {computing and even defining the Hessian is delicate}, inverting them is unthinkable to this day, first and second-order Taylor approximations are unavailable due to non-smoothness, and finally one has to deal with ``shocks'' which are inherent to inertial approaches in a non-smooth context (``corners'' and ``walls'' {in the landscape} of the loss function generate velocity discontinuity). This makes {in particular the study} of the popular algorithms ADAGRAD and ADAM in full generality quite difficult, with recent progresses reported in \cite{barakat2018convergence}.

 Our approach is inspired by the following {continuous-time} dynamical system introduced in \cite{alvarez2002second} and referred to as DIN (standing for ``dynamical inertial Newton''):
 \begin{align}
 \label{eq:physicalIntuitionSmooth}
 \underbrace{\ddot\theta(t)}_{\text{\rm Inertial term}}+\underbrace{\alpha\,\dot\theta(t)}_{\text{\rm Friction term}}+\underbrace{\beta\,\nabla^2\J(\theta(t))\dot\theta(t)}_{\text{\rm Newtonian effects}}+\underbrace{\nabla \J(\theta(t))}_{\text{\rm Gravity effect}}= \ 0, \quad \text{for $t\in [0,+\infty)$},
 \end{align}
 where $t$ is the time parameter which acts as a continuous epoch counter, $\J$ is a given loss function (e.g., the empirical loss in DL applications), for now assumed $C^2$ (twice-differentiable), with gradient $\nabla \J$ and Hessian $\nabla^2\J$. \new{It can be shown that solutions to \eqref{eq:physicalIntuitionSmooth} converge to critical points of $\cal{J}$ \citep{alvarez2002second}. As such the discretization of \eqref{eq:physicalIntuitionSmooth} can support the design of algorithms that optimize $\cal{J}$ and that leverage inertial properties with Newton's method. To adapt this dynamics to DL we must first overcome the computational or conceptual difficulties raised by the second-order objects $\ddot \theta$ and $\nabla^2\J(\theta)$ appearing in \eqref{eq:physicalIntuitionSmooth}. To do this, we propose in this paper to combine a phase-space lifting method introduced in \cite{alvarez2002second} with the use of Clarke subdifferential $\partial  \J$. Clarke subdifferential defines a notion of differentiability for non-convex and non-smooth functions. This approach  results in the study of a first-order differential inclusion in place of \eqref{eq:physicalIntuitionSmooth}, namely,
\begin{equation}
		\begin{cases}
		\dt(t) + \beta  \partial \J(\theta(t)) &+(\alpha -\ovb)\theta(t) + \ovb \psi(t) \ni 0\\
		\ddp(t) &+(\alpha -\ovb)\theta(t) + \ovb \psi(t) \ni 0  \end{cases} \mbox{,\quad for a.e. $t\in (0,+\infty)$}.
		\end{equation}
This differential inclusion can then be discretized to obtain the practical algorithm INNA that is introduced in Section~\ref{sec:algo}, together with a rigorous presentation of the concepts mentioned above.}

\new{Computation of the (sub)gradients and convergence proofs (in batch or mini-batch settings) typically rely on the sum-rule in smooth or convex settings, i.e., $\partial ({\cal J}_1 + {\cal J}_2) = \partial{\cal J}_1 + \partial{\cal J}_2$. Unfortunately this sum-rule does not hold in general in the non-convex setting using the standard Clarke subdifferential. Yet, many DL studies ignore the failure of the sum rule: they use it in practice, but circumvent the theoretical problem by modeling their method through simple dynamics (e.g., smooth or convex). We tackle this difficulty as is, and show that such practice  can create additional spurious stationary points that are not Clarke-critical. To address this question, we introduce the notion of $D$-criticality. It is less stringent than Clarke-criticality and   it describes more accurately real-world implementation. We  then show convergence of INNA to such $D$-critical points. Our theoretical results are general, simple and allow for aggressive step-sizes in $o(1/\log k)$.} {We first provide adequate calculus rules and tame non-smooth Sard's-like results for the new steady states we introduced. We then combine these results with a Lyapunov analysis from \cite{alvarez2002second} and} the differential inclusion approximation method \citep{benaim2005stochastic} to characterize the asymptotics of our algorithm similarly to \cite{davis2018stochastic,adil}. This provides a strong theoretical ground to our study since we can prove that our method converges to a connected component of the set of steady states {even for networks with ReLU or other non-smooth activation functions}. For the smooth deterministic dynamics, we also show that convergence in values is of the form $O(1/t)$ where $t$ is the running time. 
For doing so, we provide a general result  for the solutions of a family of differential inclusions having a certain type of favorable Lyapunov functions.

Our algorithm INNA shows great efficiency in practice. It has similar computational complexity to state-of-the-art methods SGD, ADAGRAD and ADAM, often achieves better training accuracy and shows good robustness to hyper-parameters selection. INNA can avoid parasitic oscillations and produce acceleration; a first illustration of the behavior of the induced dynamics is given in Figure \ref{fig::rosenbrockexp} for a simple non-smooth and non-convex function in $\R^2$.

The rest of the paper is organized as follows. INNA is introduced in details in Section \ref{sec:algo} and its convergence is established in Section \ref{sec:proof}. Convergence rates of the underlying continuous-time differential inclusion are obtained in Section~\ref{sec::rateofCV}. Section \ref{sec:numerics} describes experimental DL results on synthetic and real data sets (MNIST, CIFAR-10, CIFAR-100).

	\begin{figure}[t]
    \begin{center}%
    \begin{tabular}{cc}
       \includegraphics[width=0.47\textwidth]{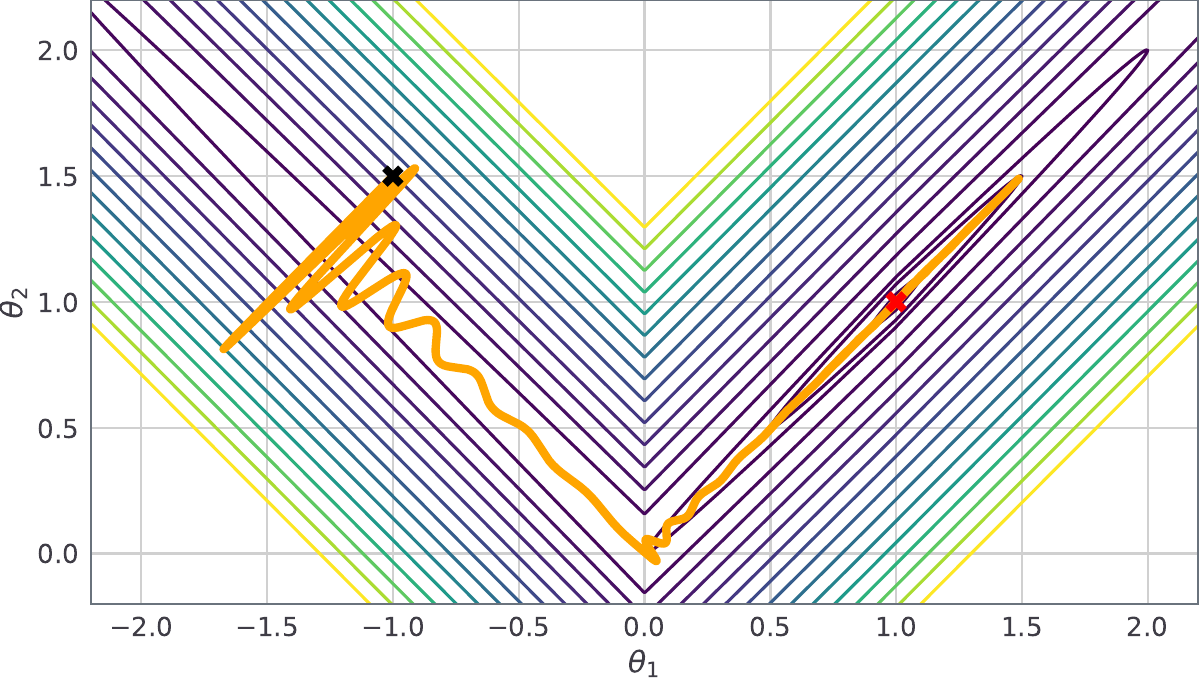} & \includegraphics[width=0.47\textwidth]{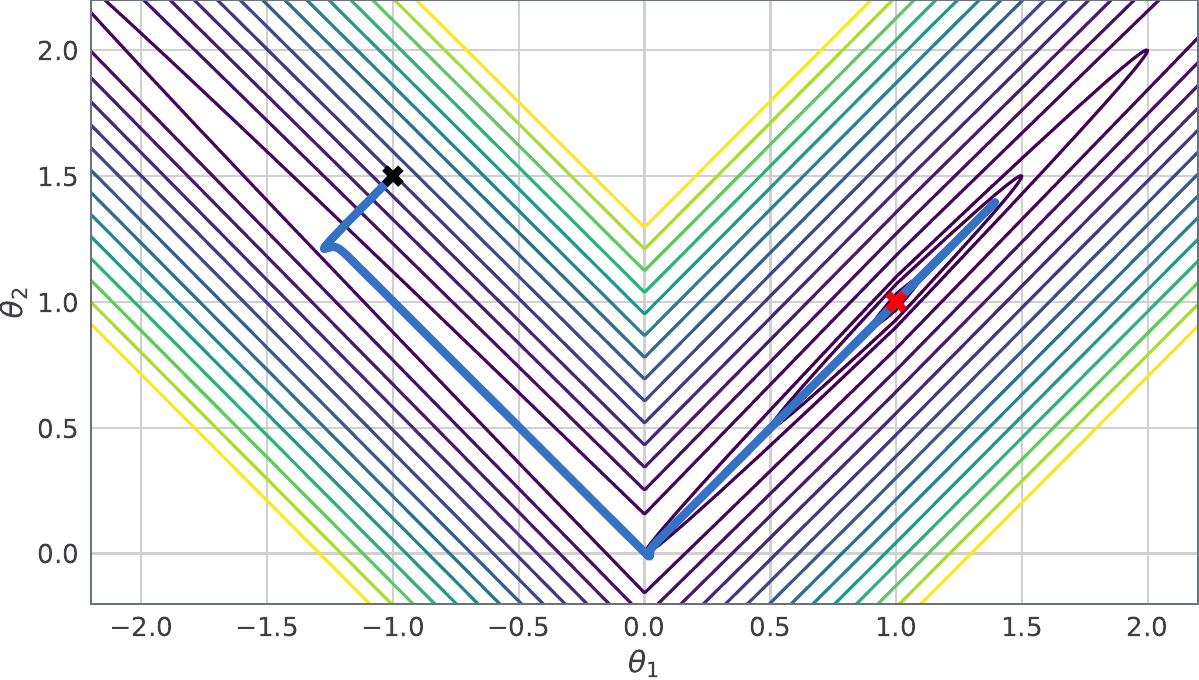}\\
       (a) $\alpha=0.5,\ \beta=0.01$  & (b) $\alpha=0.5,\ \beta=0.1$ \medskip \\
       \includegraphics[width=0.47 \textwidth]{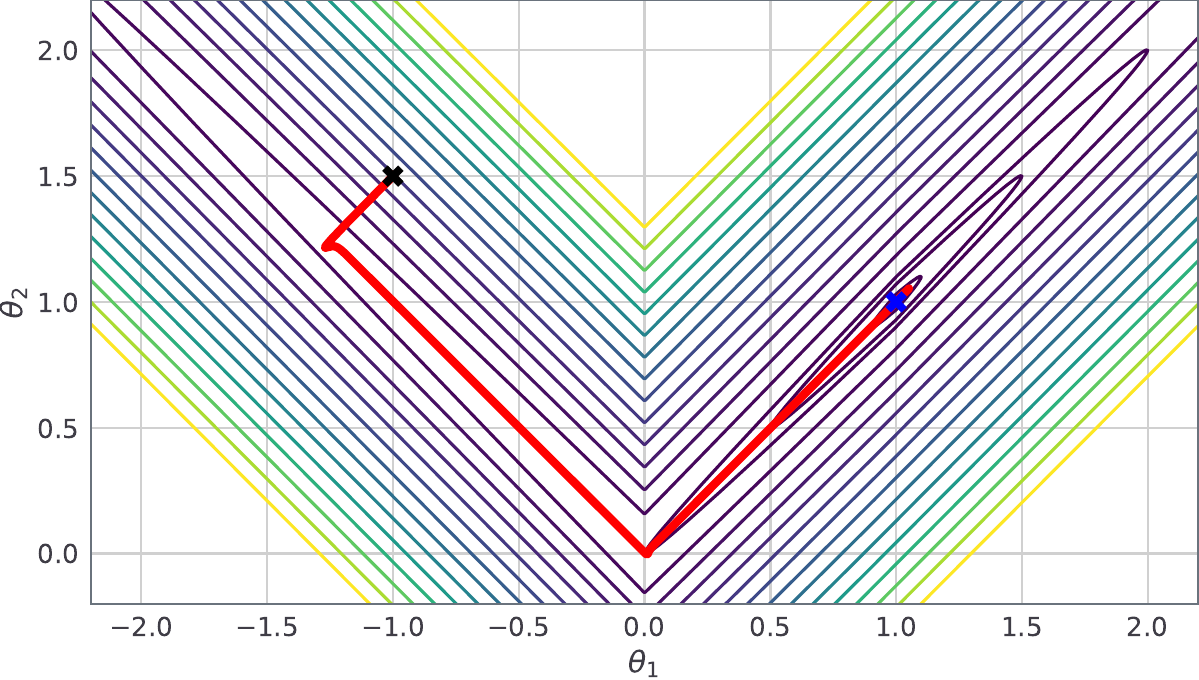} & \includegraphics[width=0.47\textwidth]{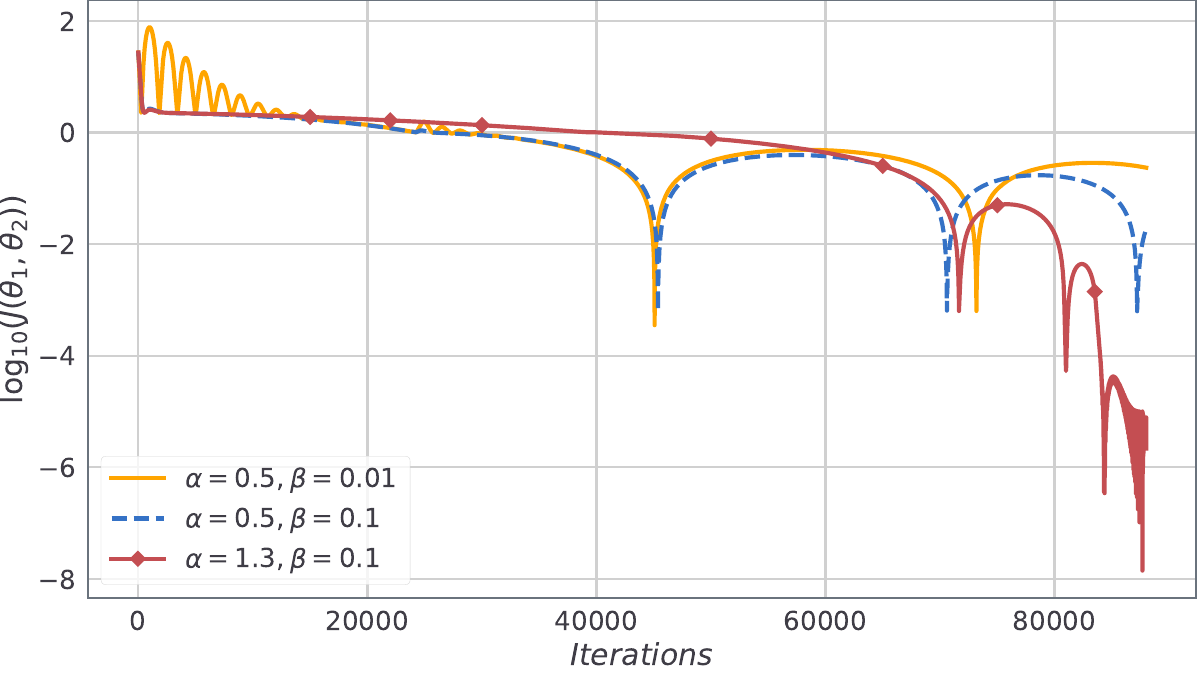}\\
       (c) $\alpha=1.3,\ \beta=0.1$ & (d) Objective function $\J$ (in log-scale)
    \end{tabular}
    \caption{Illustration of INNA applied to the non-smooth function $\J(\theta_1,\theta_2)= 100(\theta_2-\vert \theta_1 \vert)^2 +\vert 1-\theta_1\vert $. Subplots (a-c) represent the trajectories of the parameters $\theta_1$ and $\theta_2$ in $\R^2$ for three choices of hyper-parameters $\alpha$ and $\beta$, see \eqref{eq:physicalIntuitionSmooth} for an intuitive explanation. Subplot (d) displays the values of the objective function $\J(\theta_1,\theta_2)$ for the three settings considered.}\label{fig::rosenbrockexp}%
    \end{center}%
    \end{figure}

\section{INNA: an Inertial Newton Algorithm for Deep Learning}
	\label{sec:algo}
	We first introduce our functional framework, then we describe step by step the process of building INNA from \eqref{eq:physicalIntuitionSmooth}.
	\subsection{Neural Networks with Lipschitz Continuous Prediction Function and Loss Functions}\label{sec::structofDL}
    
    We consider DNNs of a general type represented by a function $f:(x,\theta)\in \R^M\times \R^P\mapsto y\in \R^D$ that is {\em locally Lipschitz continuous} in $\theta$. This includes for instance feed-forward, convolutional or residual networks used with ReLU, sigmoid, or tanh activation functions. Recall that a function $F:\R^P\to \R$ is {locally Lipschitz continuous}, if for any $\theta \in\R^P$, there exists a neighborhood $V$ of $\theta$ and a constant $C>0$ such that for any $\theta_1,\theta_2\in V$,
        \begin{equation*}
            \vert F(\theta_1) - F(\theta_2)\vert \leq C \left\|\theta_1 - \theta_2\right\|,
        \end{equation*}
        where $\left\| \cdot \right\|$ is any norm on $\R^P$. A function $F:\R^P\to \R^D$ is locally Lipschitz continuous if each of its coordinates is locally Lipschitz continuous. The variable $\theta \in \R^P$ is the parameter of the model ($ P $ can be very large), while $ x\in \R^M $ and $y\in \R^D$ represent input and output data. For instance, the vector $ x $ may embody an image while $y$ is a label explaining its content. Consider further a data set of $ N$ samples $ (x_n,y_n)_{n=1,\ldots,N} $. Training the network amounts to finding a value of the parameter $ \theta $ such that, for each input data $ x_n $ of the data set, the output $ f(x_n,\theta)$ of the model predicts the real value $ y_n $ with good accuracy. To do so, we follow the traditional approach of minimizing an empirical risk loss function, \begin{equation}\label{eq::loss}
	\R^P\ni\theta \mapsto \J(\theta)=\sum_{n=1}^N l(f(x_n,\theta),y_n),
	\end{equation} where $l:\R^D\times \R^D\to \R$ is a locally Lipschitz continuous dissimilarity measure. \new{In the sequel, for $n\geq1$, we will sometimes denote by $\J_n$ the $n$-th term of the sum: $\J_n(\theta) \triangleq l(f(x_n,\theta),y_n)$, so that $\J=\sum_{n=1}^N \J_n$. Despite the non-smoothness and the non-convexity of typical DL loss functions, they generally possess a very strong  property sometimes called tameness. We now introduce this notion which is essential to obtain the convergence results of Section~\ref{sec:proof}.}

	\subsection{Neural Networks and Tameness in a Nutshell}\label{sec::favstruct}

Tameness  refers to a ubiquitous geometrical property of loss functions and constraints encompassing most finite dimensional optimization problems met in practice. Prominent classes of tame objects are piecewise-linear or piecewise-polynomial objects (with finitely many pieces), and more generally, semi-algebraic objects. However, the notion is much more general, as we intend to convey below. The formal definition is given at the end of this subsection (Definition~\ref{DEF}).

Informally, sets or functions are called  tame  when they can be described by a finite number of basic formulas, inequalities, or Boolean operations involving standard functions such as polynomial, exponential, or max functions. We refer to \cite{attouch} for illustrations, recipes and examples within a general optimization setting or  \cite{davis2018stochastic} for illustrations in the context of neural networks. The reader is referred to \cite{van1998tame,coste2000introduction,shiota} for foundational material. To apprehend the strength behind tameness it is convenient to remember that it models non-smoothness by confining the study to sets and functions which are union of smooth pieces. This is the so-called {\em stratification} property of tame sets and functions. It was this property which motivated the term of {\em tame topology},\footnote{``{\em La topologie mod\'er\'ee}'' wished for by Grothendieck.} see \cite{van1998tame}. In a non-convex optimization settings, the stratification property is crucial to generalize qualitative algorithmic results to non-smooth objects.

{\em Most finite dimensional DL optimization models we are aware of yield tame loss functions $\J$}. To understand this assertion and illustrate the wide scope of tameness assumptions, let us provide concrete examples (see also \citealt{davis2018stochastic}).
Assume that the DNNs under consideration are built from the following traditional components:
\begin{itemize}
    \renewcommand\labelitemi{--}
    \item the network architecture describing $ f $ is fixed with an arbitrary number of layers of arbitrary dimensions and arbitrary Directed Acyclic Graph (DAG) representing computations,
\item  the activation functions are among classical ones: ReLU, sigmoid, SQNL, RReLU, tanh, APL, soft plus, soft clipping, and many others including multivariate activation functions (norm, sorting), or activation functions defined piecewise with polynomials, exponential and logarithm,
\item  the dissimilarity function $l$ is a standard loss such as $\ell_p$ norms, logistic loss or cross-entropy, or more generally a function defined piecewise using polynomials, exponentials and logarithms,
\end{itemize}
then one can easily show, by elementary quantifier elimination arguments (property (iii) below), that the corresponding loss, $\J$,  is tame.
\medskip

For the sake of completeness, we provide below the formal definition of tameness and o-minimality. 
\begin{definition}\label{DEF}
{\rm {\bf [o-minimal structure] }}{\rm
\cite[Definition\,1.5]{coste2000introduction}} \label{Domin}\rm{An {\it
o-minimal } structure on $(\R,+,.)$ is a countable collection of sets ${\cal O}=\{\mathcal{O}_{q}\}_{q\geq 1}$ where each $\mathcal{O}_{q}$ is itself a collection of subsets of $\R^q$, called {\em definable} subsets. They must have the following properties,  for each $q\geq 1$:}
\begin{enumerate}\itemsep=1mm
\item[(i)] (Boolean properties) $\mathcal{O}_{q}$ contains the empty set, is stable by finite union, finite intersection and complementation;
\item[(ii)] (Lifting property) 
{\rm if $A$ belongs to $\mathcal{O}_{q}$, then $A\times\R$ and
$\R\times A$ belong to $\mathcal{O}_{q+1}$.}
\item[(iii)] (Projection or quantifier elimination property) 
{\rm if $\Pi:\R^{q+1}\rightarrow\R^q$ is the canonical
projection onto $\R^q$ then for any $A$ in $\mathcal{O} _{q+1}$,
the set $\Pi(A)$ belongs to $\mathcal{O}_{q}$.}
\item[(iv)] (Semi-algebraicity)
{\rm $\mathcal{O}_{q}$ contains the family of algebraic
subsets of $\R^q$, that is, every set of the form
\[
\{\theta\in\R^q\mid\zeta(\theta)=0\},
\]
where $\zeta:\R^q\rightarrow\R$ is a polynomial function.}
\item[(v)] (Minimality property), 
{\rm  the elements of $\mathcal{O}_{1}$ are exactly the finite
unions of intervals and points. }
\end{enumerate}
\end{definition}

\noindent A mapping $F:S\subset \R^m\rightarrow \R^q$ is said to be
{\em definable in ${\cal O}$} if its graph is  definable in $\cal O$ as a subset of~$\R^m\times\R^q$. For illustration of o-minimality in the context of optimization one is referred to \cite{attouch,davis2018stochastic}.

\begin{center}
\fbox{\begin{minipage}{14cm}  From now on we fix an o-minimal structure $\mathcal{O}$ and a set or a mapping definable in $\cal O$ will be called {\em tame}.
\end{minipage}}
\end{center}

\subsection{From DIN to INNA}\label{sec::optpart}

        We describe in this section the construction of our proposed algorithm INNA from the discretization of the second-order ODE~\eqref{eq:physicalIntuitionSmooth}.

		\subsubsection{Handling Non-smoothness and Non-convexity}\label{sec::non-smoothnonconv}

		 We first show how the formalism offered by Clarke's subdifferential can be applied to generalize \eqref{eq:physicalIntuitionSmooth} to the non-smooth non-convex setting. Recall that the dynamical system \eqref{eq:physicalIntuitionSmooth} is described by,
		\begin{equation}\label{eq::secondorderDIN}
		\ddot{\theta}(t)+\alpha\dt(t) + \beta \nabla^2 \J(\theta(t))\dt(t) +\nabla \J(\theta(t))=0,
		\end{equation}
		where $\J$ is a twice-differentiable potential,  $\alpha>0$, $\beta>0$ are two hyper-parameters and $\theta: \mathbb{R}_+ \rightarrow \mathbb{R}^P$. We cannot exploit \eqref{eq::secondorderDIN} directly since in most DL applications $\J$ is not twice differentiable (and even not differentiable at all). We first overcome the explicit use of the Hessian matrix $\nabla^2 \J$ \new{by introducing an auxiliary variable $\psi: \mathbb{R}_+ \rightarrow \mathbb{R}^P$ like in \cite{alvarez2002second}. Consider the following dynamical system (defined for $\J$ merely differentiable),}
		\begin{equation}\label{eq::contdinSmooth}
		\begin{cases}
		\dt(t) + \beta  \nabla \J(\theta(t)) &+(\alpha -\ovb)\theta(t) + \ovb \psi(t) = 0\\
		\ddp(t) &+(\alpha -\ovb)\theta(t) + \ovb \psi(t) = 0  \end{cases} \mbox{,\quad for a.e. $t\in (0,+\infty)$}.
		\end{equation}
		\new{As explained in \cite{alvarez2002second}, \eqref{eq::secondorderDIN} is equivalent to \eqref{eq::contdinSmooth} when $\J$ is twice differentiable. Indeed, one can rewrite \eqref{eq::secondorderDIN} into \eqref{eq::contdinSmooth} by introducing $\psi = -\beta\dot{\theta}-\beta^2\nabla\J(\theta)-(\alpha\beta-1)\theta$. Conversely, one can substitute the first line of \eqref{eq::contdinSmooth} into the second one to retrieve \eqref{eq::secondorderDIN}.} Note however that \eqref{eq::contdinSmooth} does not require the existence of second-order derivatives.
		
		Let us now introduce a new {\em non-convex non-differentiable} version of \eqref{eq::contdinSmooth}.
		By Rademacher's theorem, locally Lipschitz continuous functions $\J:\R^P\to \R$ are differentiable almost everywhere. Denote by $\mathsf{R}$ the set of points where $\J$ is  differentiable. Then, $\R^P\setminus \mathsf{R}$ has zero Lebesgue measure. It follows that for any $\theta^\star\in\R^P\setminus \mathsf{R}$, there exists a sequence of points in $\mathsf{R}$ whose limit is this $\theta^\star$. This motivates the introduction of the subdifferential due to \cite{clarke1990optimization}, defined next.
		\begin{definition}[Clarke subdifferential of Lipschitz functions] \label{def:clarke}
			For any locally Lipschitz continuous function $F: \R^P\to \R$, the Clarke subdifferential of $F$ at $\theta\in\R^P$, denoted $\partial F(\theta)$, is the set defined by,
			\begin{equation}
				\partial F(\theta) = \mathrm{conv}\left\{ v\in\R^P \mid \exists (\theta_k)_{k\in\N}\in \mathsf{R}^\N,\text{ such that } \theta_k \xrightarrow[k\to\infty]{}\theta \text{ and }  \nabla F(\theta_k) \xrightarrow[k\to\infty]{} v \right\},
			\end{equation}
			where $\mathrm{conv}$ denotes the convex hull operator. \new{The elements of the Clarke subdifferential are called Clarke subgradients.}
		\end{definition}
		The Clarke subdifferential is a nonempty compact convex set. \new{It coincides with the gradient for smooth functions and with the traditional subdifferential for non-smooth convex functions. As already mentioned, and contrarily to the (sub)differential operator, it does not enjoy a sum rule.} 
		
		Thanks to Definition~\ref{def:clarke}, we can extend \eqref{eq::contdinSmooth} to non-differentiable functions. Since $\partial \J(\theta)$ is a set, we no longer study a differential equation but rather a {\em differential inclusion}, given by,
		 \begin{equation}\label{eq::contdinClarke}
		 \begin{cases}
		 \dt(t) + \beta  \partial \J(\theta(t)) &+(\alpha -\ovb)\theta(t) + \ovb \psi(t) \ni 0\\
		 \ddp(t) &+(\alpha -\ovb)\theta(t) + \ovb \psi(t) \ni 0  \end{cases} \mbox{,\quad for a.e. $t\in (0,+\infty)$}.
		 \end{equation}
		 For a given initial condition $(\theta_0,\psi_0)\in \R^P\times\R^P$, we call {\em solution} (or {\em trajectory}) of this system any absolutely continuous curve $(\theta,\psi)$ from $\R_+$ to $\R^P\times\R^P$ for which $(\theta(0),\psi(0)) = (\theta_0,\psi_0)$ and \eqref{eq::contdinClarke} holds. We recall that absolute continuity amounts to the fact that $\theta$ is differentiable almost everywhere with integrable derivative and,
		$$\theta(t) -\theta(0)=\int_0^t\dot\theta(s)\diff s, \mbox{ for $t\in [0,+\infty)$.}$$
		Due to the properties of the Clarke subdifferential, existence of a solution to differential inclusions such as \eqref{eq::contdinClarke} is ensured, see \cite{aubin}; note however that uniqueness of the solution does not hold in general. We will now use the structure of \eqref{eq::contdinClarke} to build a new algorithm to train DNNs.

		 \subsubsection{Discretization of the Differential Inclusion}

        To obtain the basic form of our algorithm, we discretize \eqref{eq::contdinClarke} according to the classical explicit Euler method. Given $(\theta,\psi)$ a solution of \eqref{eq::contdinClarke} and any time $t_k$, set $\theta_k = \theta(t_k)$ and $\psi_k = \psi(t_k)$. Then, at time $t_{k+1}=t_{k}+\gamma_k$ with $\gamma_k$ positive small, one can approximate $\dt(t_{k+1})$ and $\ddp(t_{k+1})$ by
        \[ \dt(t_{k+1})\simeq \frac{\theta_{k+1}-\theta_k}{\gamma_k},\quad \quad \ddp(t_{k+1})\simeq \frac{\psi_{k+1}-\psi_k}{\gamma_k}. \]
        This discretization yields the following algorithm,
        \begin{equation}\,\,\,\label{eq::notinna}
        	\begin{cases}
        	    v_k &\in \partial\J(\theta_k)\\
            	\theta_{k+1}&=  \theta_k + \gamma_k \left( (\frac{1}{\beta}-\alpha)\theta_k -\frac{1}{\beta}\psi_k - \beta v_k \right)\\
            	\psi_{k+1}& = \psi_k + \gamma_k \left( (\frac{1}{\beta}-\alpha)\theta_k - \frac{1}{\beta} \psi_k \right)
        	\end{cases}
        	\end{equation}

        Although the algorithm above is well-defined for our problem, \new{it is not suited to  DL. First, the computation of $\partial\J(\theta_k)$ is generally not possible since there is no general operational calculus for Clarke's subdifferential; secondly, a mini-batch strategy must be designed to cope with the large dimension of DL problems which makes the absence of sum rule even more critical.
        
        The next section is meant to address these issues and design  a practical algorithm.}

        \subsubsection{INNA Algorithm and a New Notion of Steady States}

        In order to compute or approximate the subdifferential of $\J$ at each iteration and to cope with large data sets, $\J$ can be approximated by mini-batches, reducing the memory footprint and computational cost of evaluation.  For any $\mb \subset \{1,\ldots, N\}$, let us define
	    \begin{align}\label{eq::minibatch}
	        \J_{\mb} \colon \theta \mapsto \sum_{n \in \mb} l(f(x_n,\theta),y_n).
	    \end{align}
        Unlike in the differentiable case, subgradients do not in general sum up to a subgradient of the sum, that is $\partial \J_\mb(\theta) \neq \sum_{n \in \mb} \partial l(f(x_n,\theta),y_n) $ in general. To see this, take for example $ 0 = |\cdot| - |\cdot|$, the Clarke subgradient of this function at $0$ is $\{0\}$, whereas $\partial(|0|) + \partial(-|0|) = [-1,1] + [-1,1] = [-2,2]$. \new{Standard DL solvers use backpropagation algorithms which implement smooth calculus on non-smooth and non-convex objects. Due to the absence of qualification conditions, {\em the resulting objects are not Clarke subgradients in general}.  In order to match the real-world practice of DL,} we introduce a notion of steady states that corresponds to the stationary points generated by a generic mini-batch approach. As we shall see, this  allows both for practical applications and convergence analysis (despite the sum rule failure for Clarke subdifferential). \new{ We emphasize once more that our goal is to {\em capture the  stationary points that are actually met in practice}.}
        
        For any $\mb \subset \{1,\ldots, N\}$, we introduce the following objects,
	    \begin{equation}
			D\J_{\mb}=\sum_{n\in \mb} \partial\left[  l(f(x_n,\cdot),y_n)\right], \quad D\J=\sum_{n=1}^N \partial\left[ l(f(x_n,\cdot),y_n)\right].
		\end{equation}
	    Observe that, for each $\mb$, we have $ D\J_{\mb}\supset \partial \J_{\mb} $ and that $\J_\mb$ is differentiable almost everywhere with $D\J_{\mb}= \partial \J_{\mb}=\{\nabla \J_{\mb}\}$, see \cite{clarke1990optimization}. In particular $D\J = \partial \J$ almost everywhere so that the potential differences with the Clarke subgradient occur on a negligible set. When $\J$ is tame the equality even hold on the complement of a finite union of manifolds of dimension strictly lower than $P$---use the classical stratification results for o-minimal structures, \cite{coste2000introduction}. A point satisfying $D\J(\theta)\ni 0$ will be called {\em $D$-critical}. \new{Note that Clarke-critical points ($0 \in \partial\J$) are  $D$-critical points but that the converse is not true.} This terminology is motivated by favorable properties: sum and chain rules along curves (see Lemmas \ref{lem::chainrule} and \ref{lem:chainRuleSum} below) and  the existence of a tame Sard's theorem (see Lemma \ref{lem:sard}). To our knowledge, this notion of a steady state has not previously been used in the literature and a direct approach modeling the mini-batch practice has never been considered before.\footnote{Following the first arXiv preprint of this work \citep{castera2019inertial}, \cite{BP} have further developed the present ideas and in particular the connection to the backpropagation algorithm.} While this notion is needed for the theoretical analysis, one should keep in mind that $D\J$ is actually what is computed numerically provided that the automatic differentiation library returns a Clarke subgradient. This computation is usually done with a backpropagation algorithm, similarly to the seminal method of \cite{rumelhart1986learning}.

	    Ultimately, one can rewrite \eqref{eq::contdinClarke} by replacing $\partial\J$ by $D\J$, which yields a differential inclusion adapted to study mini-batch approximations of non-smooth loss functions $\J$. This reads, 
	    \begin{equation}\label{eq::contdin}
	    \begin{cases}
	    \dt(t) + \beta  D \J(\theta(t)) &+(\alpha -\ovb)\theta(t) + \ovb \psi(t) \ni 0\\
	    \ddp(t) &+(\alpha -\ovb)\theta(t) + \ovb \psi(t) \ni 0  \end{cases} \mbox{,\quad for a.e. $t\in (0,+\infty)$}.
	    \end{equation}
	    Discretizing this system gives a workable version of INNA.
	    Let us consider a sequence $(\mb_k)_{k \in \N}$  of nonempty  subsets of $\{1,\ldots,N\}$, chosen independently and uniformly at random with replacement, and a sequence of positive step-sizes $(\gamma_k)_{k \in \N}$.
	    For a given initialization $(\theta_0,\psi_0)\in\R^P\times\R^P$, at iteration $k\geq 1$, our algorithm reads,
	    	\begin{equation}\hspace{-2cm}\textrm{(INNA)}\,\,\,\label{eq::discdin}
        	\begin{cases}
        	    v_k &\in D \J_{\mb_k}(\theta_k)\\
            	\theta_{k+1}&=  \theta_k + \gamma_k \left( (\frac{1}{\beta}-\alpha)\theta_k -\frac{1}{\beta}\psi_k - \beta v_k \right)\\
            	\psi_{k+1}& = \psi_k + \gamma_k \left( (\frac{1}{\beta}-\alpha)\theta_k - \frac{1}{\beta} \psi_k \right)
        	\end{cases}
        	\end{equation}
        Here again $\alpha > 0$ and $\beta > 0$ are hyper-parameters of the algorithm. {The mini-batch procedure forms} a stochastic approximation of the deterministic dynamics obtained by choosing $\mb_k=\{1,\ldots,N\}$, i.e., when $\J_{\mb_k}=\J$ (batch version). This can be seen by observing that the vectors $v_k$ above may be written as $v_k = \tilde{v}_k+\eta_k$, where $\tilde{v}_k \in D\J(\theta_k)$ and $\eta_k$ compensates for the missing subgradients and can be seen as a zero-mean noise. 	
	Hence, INNA admits the following general abstract stochastic formulation,
	\begin{equation}\,\,
        \begin{cases}
        	  w_k &\in D \J(\theta_k)\\
            \theta_{k+1}&=  \theta_k + \gamma_k \left( (\frac{1}{\beta}-\alpha)\theta_k -\frac{1}{\beta}\psi_k - \beta w_k +\xi_k\right)\\
            \psi_{k+1}& = \psi_k + \gamma_k \left( (\frac{1}{\beta}-\alpha)\theta_k - \frac{1}{\beta} \psi_k \right)
        	\end{cases}\label{eq::innag}
    \end{equation}
 where $(\xi_k)_{k \in \N}$ is a martingale difference noise sequence adapted to the filtration induced by (random) iterates up to $k$. While~\eqref{eq::discdin} is the version implemented in practice, its equivalent form~\eqref{eq::innag} is more convenient for the convergence analysis of the next section. We stress that the equivalence between \eqref{eq::discdin} and \eqref{eq::innag} relies on the use of $D\J$ and would not hold with the use of $\partial \J$ as in~\eqref{eq::notinna}.

INNA in its general and practical form is summarized in Table~\ref{tab:inna}.

 \begin{table}[ht]
    \centering
\begin{mdframed}[style=MyFrame]
 \begin{center}
 {\bf Inertial Newton Algorithm for Deep Learning (INNA)}
 \end{center}
\bigskip

\noindent
\textbf{Objective function:} $\J = \sum_{n = 1}^N \J_n$, with $\J_n\colon \R^P \mapsto \R$ locally Lipschitz.\\
\textbf{Hyper-parameters:} $(\alpha,\beta)$ positive.\\
\textbf{Mini-batches:} $(\mb_k)_{k \in \N}$, nonempty  subsets of $\{1,\ldots,N\}$.\\
\textbf{step-sizes:} $(\gamma_k)_{k \in \N}$ positive.\\
\textbf{Initialization:}  $(\theta_0,\psi_0)\in\R^P\times\R^P$.\\

\textbf{For $k\in \N$:}
\begin{equation*}
        	\begin{cases}
        	    v_k &\in\quad \sum_{n\in \mb_k} \partial\left[  \J_n(\theta_k)\right]\\
            	\theta_{k+1}&=\quad  \theta_k + \gamma_k \left( (\frac{1}{\beta}-\alpha)\theta_k -\frac{1}{\beta}\psi_k - \beta v_k \right)\\
            	\psi_{k+1}& =\quad \psi_k + \gamma_k \left( (\frac{1}{\beta}-\alpha)\theta_k - \frac{1}{\beta} \psi_k \right)
        	\end{cases}
\end{equation*}
\end{mdframed}
\caption{INNA in a nutshell. \label{tab:inna}}
\end{table}

    \section{Convergence Results for INNA}
\label{sec:proof}

    We first state our main result.
	\subsection{Main result: Accumulation Points of INNA are Critical}

We now study the convergence of INNA. The main idea here is to prove that the discrete algorithm~\eqref{eq::discdin} asymptotically behaves like the solutions of the continuous differential inclusion~\eqref{eq::contdin}. In addition to tameness, our main assumption is the following:
	\begin{assumption}[Stochastic approximation]\label{ass:itbound}
	    The sets $(\mb_k)_{k \in \N}$ are taken independently uniformly at random with replacement.
		The step-size sequence $\gamma_k$ is positive with $\sum_k\gamma_k=+\infty$ and satisfies $\gamma_k = o \left( \frac{1}{\log k}\right)$, that is  $\displaystyle\limsup_{k\to+\infty} |\gamma_k \log k|=0$.	       \label{ass:mainAssumption}
	\end{assumption}
Typical admissible choices are $\gamma_k=C(k+1)^{-a}$ with $a\in (0,1]$, $C > 0$. 
	The main theoretical result of the paper follows. 

	\begin{theorem}[INNA converges to the set of $D$-critical points of $\J$]\label{th::thmDIN}
		Assume that for \new{$n\in\{1,\ldots,N\}$, each $\J_n$ is locally Lipschitz continuous}, tame and that the step-sizes satisfy Assumption~\ref{ass:mainAssumption}. Set an initial condition $(\theta_0,\psi_0)$ and assume that there exists $M>0$ such that $\sup_{k\geq 0}\|(\tk,\pk)\|\leq M$ almost surely, where $(\theta_k,\psi_k)_{k\in\N}$ are generated by INNA.
	    Then, almost surely, any accumulation point $\bar\theta$ of the sequence $(\theta_k)_{k\in \N}$ satisfies $D\J(\bar\theta)\ni 0 $. In addition $ (\J(\theta_{k}))_{k\in\N}$ converges.	    
	\end{theorem}

	\new{Before proving Theorem~\ref{th::thmDIN}, we will first make some comments and illustrate this result.}
    \subsection{Comments on the Results of Theorem~\ref{th::thmDIN}}\label{sec::comRes}
    
    \begin{itemize}
        \renewcommand\labelitemi{--}
        \item {On the step-sizes.}
        First, Assumption~\ref{ass:mainAssumption} offers much more flexibility than the usual $O(1/\sqrt{k})$ assumption commonly used for SGD. We leverage the boundedness assumption on the norms of $(\theta_k,\psi_k)$, the local Lipschitz continuity and finite-sum structure of $\J$, so that the noise is actually uniformly bounded and hence sub-Gaussian, allowing for much larger step-sizes than in the more common bounded second moment setting. See \citet[Remark 1.5]{benaim2005stochastic} and \cite{benaim1999dynamics} for more details. The interest of this aggressive strategy is highlighted in Figure~\ref{fig::decay} of the experimental section. 
        
        \item {On the scope of the theorem.}
        Our result actually holds for general locally Lipschitz continuous tame functions with finite-sum structure and for the general stochastic process under uniformly bounded martingale increment noise. We do not use any other specific structure of DL loss functions. Other variants could be considered depending on the assumptions on the noise, see \cite{benaim2005stochastic}.

        \item {On $D$-criticality.}
        The result of Theorem~\ref{th::thmDIN} states that the bounded discrete trajectories of INNA are attracted by the $D$-critical points.  Recall that $D$-critical points include local minimizers and thus our theoretical finding agrees with our empirical observations that most initializations lead to ``valuable weights'' $\theta$ and to efficient training. In particular for smooth networks where $\J$ is differentiable, limit points of INNA are simply critical points of $\cal J$.
        The reader should however remember that when the algorithm is initialized on the $D$-critical set, the algorithm is stationary as well, {\em even when the initialization is non-Clarke critical}.  This last point shows that $D$-points are not introduced to simplify the analysis but to {\em sharply model the use of mini-batch methods on non-convex and non-smooth problems}.  Hopefully, in practice one can expect to avoid such points with overwhelming probability. 
        Indeed, \citet{bolte2020mathematical} proved  that SGD converges with probability one to the set of Clarke-critical points. In other words,  $D$-critical points that are not Clarke-critical are not seen by the dynamics (see also \citealt{bianchi2020convergence}). The key argument is to prove that the set of initializations and step-sizes that leads the algorithm to reach points where $D\J\neq\partial\J$ has zero-measure. The same result can be hoped for INNA.

      \item  {On the boundedness assumption.} The bounded assumption on the iterates is a classical assumption for first or second-order algorithms, see for instance \cite{davis2018stochastic,duchi2018stochastic}. When using deterministic algorithms (i.e., without mini-batch approximations), properties such as the coercivity of $\J$ can be sufficient to remove the boundedness assumption for descent algorithms. This does not remain true when dealing with mini-batch approximations, yet, in the case of INNA, the coercivity of $\J$ would guarantee at least that the solutions of the continuous underlying differential inclusion \eqref{eq::contdin} remain bounded. Indeed, we will prove in Section~\ref{sec::proofofConv} that for any solution $(\theta,\psi)$ of \eqref{eq::contdin}, the function $E(\theta(t),\psi(t)) \triangleq 2(1+\alpha\beta)\J(\theta(t)) + \left\Vert \aaa\theta(t) +\bb \psi(t) \right\Vert^2 $ is decreasing in time (see Lemma~\ref{lem::Edec} hereafter). As a consequence, we cannot have $\J(\theta(t))\xrightarrow[t\to\infty]{}\infty$ so $\Vert\theta(t)\Vert\not\to\infty$ due to the coercivity of $\J$. In addition this guarantees $\Vert\psi(t)\Vert\not\to\infty$ as well. However, DL loss functions are not coercive in general and studying the boundedness issue in DL or even for non-convex semi-algebraic problems is far beyond the scope of this paper. Let us however mention that it is not uncommon to project the iterates on a given large ball to ensure boundedness;  this is a matter for future research.
  \end{itemize}
  
 \subsection{Preliminary Variational Results}
	\label{sec:ResultsOnD}
    Prior to proving Theorem 3, we extend some results known for the Clarke subdifferential of tame functions to the operator $D$ that we previously introduced. First, we recall a useful result of \cite{davis2018stochastic} which follows from the projection formula in \cite{bolte2007clarke}.
	\begin{lemma}[Chain rule for the Clarke subdifferential]\label{lem::chainrule}
	    Let $\J:\R^P\to \R$ be a locally Lipschitz continuous tame function, then $\J$ admits a chain rule, meaning that for all absolutely continuous curves $\theta:\R_+\to\R^P$, \new{$\J\circ\theta$ is differentiable a.e.} and for a.e. $t\geq 0$,
	    \begin{equation}\label{eq::chainruleJ}
	        \frac{\diff \J}{\diff t} (\theta(t)) = \langle \dot{\theta}(t),\partial \J(\theta(t)) \rangle =  \langle \dot{\theta}(t) ,v \rangle,\ \ \forall v\in \partial \J(\theta(t)).
	    \end{equation}
	\end{lemma}
\medskip
\new{Note that, even though $\J$ is non-differentiable on $\R^P$, the function $t\mapsto \J(\theta(t))$ is differentiable for a.e. $t>0$. Indeed, as introduced in Section~\ref{sec::non-smoothnonconv}, an absolutely continuous curve from $t\geq 0$ to $\R^P$ is differentiable for a.e. $t>0$. This, combined with the chain-rule of Lemma~\ref{lem::chainrule} allows differentiating $\J\circ\theta$ for a.e. $t>0$ whenever $\J$ is tame and locally Lipschitz continuous. \new{Additionally, notice that the value of $\frac{\diff \J}{\diff t} (\theta(t))$ in \eqref{eq::chainruleJ} does not depend on the element $v$ taken in $\partial\J(\theta(t))$ which justifies the notation $\langle \dot{\theta}(t),\partial \J(\theta(t)) \rangle$}.}

Consider now a function with an additive finite-sum structure (such as in DL):
\begin{align}\label{eq::fsumofTame}
		\J \colon\R^P \ni \theta \mapsto \sum_{n=1}^N \J_n(\theta),
\end{align}
where each $\J_n \colon \R^P \mapsto \R$ is locally Lipschitz continuous and tame. We define for any $\theta \in \R^P$
\begin{align*}
				D\J(\theta) = \sum_{n=1}^N \partial \J_n(\theta).
\end{align*}
The following lemma is a direct generalization of the above chain rule.
\begin{lemma}[Chain rule for $D\J$]\label{lem::chainruleD}
				Let $\J$ be a sum of tame functions like in \eqref{eq::fsumofTame}. Let $c \colon [0,1] \mapsto \R^P$ be an absolutely continuous curve so that $t \mapsto \J(c(t))$ is differentiable almost everywhere. For a.e. $t \in [0,1]$, and for all $v \in D\J(c(t))$,
				\begin{align*}
								\frac{\diff}{\diff t} \J(c(t)) = \left\langle v, \dot{c}(t) \right\rangle.
				\end{align*}
				\label{lem:chainRuleSum}

\end{lemma}
    \begin{proof}
				By local Lipschitz continuity and absolute continuity, each $\J_n$ is differentiable almost everywhere and Lemma \ref{lem::chainrule} can be applied:
				\begin{align*}
								\frac{\diff}{\diff t} \J_n(c(t)) = \left\langle v_n, \dot{c}(t) \right\rangle, \mbox{for all $v_n \in \partial \J_n(c(t))$ and for a.e. $t\geq 0$.}
				\end{align*}
		  Thus
				\begin{align*}
								\frac{\diff}{\diff t} \J(c(t)) = \sum_{n=1}^N \frac{\diff}{\diff t} \J_n(c(t)) = \sum_{n=1}^N\left\langle v_n, \dot{c}(t) \right\rangle,
				\end{align*}
			for any $v_n \in \partial \J_n(c(t))$, for all $n=\{1,\ldots, N\}$, and for a.e. $t\geq 0$. This proves the desired result.
    \end{proof}
    We finish this section with a Sard lemma for $D$-critical values, in the spirit of \cite{bolte2007clarke}.
    \begin{lemma}[A Sard's theorem for tame D-critical values]
				Let, $$ \mathsf{S}= D\mbox{\rm -crit\,}\triangleq\left\{ \theta \in \R^P\mid D\J(\theta) \ni 0 \right\},$$ then $\J(\mathsf{S})$ is finite. 
				\label{lem:sard}
    \end{lemma}
    \begin{proof}
				The set $\mathsf{S}$ is tame and hence it has a finite number of connected components. It is sufficient to prove that $\J$ is constant on each connected component of $\mathsf{S}$. Without loss of generality, assume that $\mathsf{S}$ is connected and consider $\theta_0,\theta_1 \in \mathsf{S}$. By Whitney regularity \cite[4.15]{van1998tame}, there exists a tame continuous path $\Gamma$ joining $\theta_0$ to $\theta_1$. Because of the tame nature of the result, we should here conclude with only tame arguments and use the projection formula in \cite{bolte2007clarke}, but for convenience of readers who are not familiar with this result we use  Lemma \ref{lem::chainrule}. Since $\Gamma$ is tame, the monotonicity lemma (see for example \citealt[Lemma 2]{kurdyka1998gradients}) gives the existence of a finite collection of real numbers $0 = a_0 < a_1 < \ldots < a_q = 1$, such that $\Gamma$ is $C^1$ on each segment $(a_{j-1}, a_{j})$, $j = 1,\ldots,q$. Applying Lemma \ref{lem::chainrule} to each $\Gamma_{|(a_i,a_{i+1})}$, we see that $\J$ is constant except perhaps on a finite number of points, thus $\J$ is constant by continuity.
    \end{proof}

	\subsection{Proof of Convergence for INNA}
	\label{sec::proofofConv}
  Our approach follows the stochastic method for differential inclusions developed in \cite{benaim2005stochastic} for which the differential system \eqref{eq::contdin}  and its Lyapunov properties play fundamental roles. 
	The steady states of \eqref{eq::contdin}\, are given by,
	\begin{equation}\label{S} \ma = \left\{ (\theta,\psi)\in \R^P\times \R^P \mid 0\in D \J(\theta), \psi=(1-\alpha\beta) \theta \right\}.\end{equation}
	These points are initialization values for which the system does not evolve and remains constant.
    Observe that the first coordinates of these points are $D$-critical for $\J$ and that conversely any $D$-critical point of $\J$ corresponds to a unique rest point in $\ma$.

	\begin{definition}[Lyapunov function]\label{def::lyap}
		Let $ \mathsf{A} $ be a subset of $\R^P \times \R^P $, we say that $ E : \R^P\times \R^P\to \R $ is a Lyapunov function for the set $ \mathsf{A} $ and the dynamics \eqref{eq::contdin} if,
		\renewcommand{\theenumi}{(\roman{enumi})}%
		\begin{enumerate}
		    \item For any  solution $ (\theta,\psi)$ of  \eqref{eq::contdin} with initial condition $(\theta_0,\psi_0)\in\R^P\times\R^P $, we have:\\
		    $E(\theta(t),\psi(t) ) \leq E(\theta_0,\psi_0)$ a.e. on $\R$.
		     \item For any  solution $ (\theta,\psi)$ of  \eqref{eq::contdin} with initial condition $(\theta_0,\psi_0) \in\R^P\times\R^P\setminus\mathsf{A}$, we have:\\
		     $E(\theta(t),\psi(t) ) < E(\theta_0,\psi_0)$ a.e. on $\R$.
		\end{enumerate}
	\end{definition}

	In practice, to establish that a functional is Lyapunov, one can simply use differentiation through chain rule results, with in particular Lemma \ref{lem::chainrule}. In the context of INNA, we will use Lemma \ref{lem:chainRuleSum}. To build a Lyapunov function for the dynamics \eqref{eq::contdin} and the set $ \ma $, consider the two following energy-like functions,
	\begin{equation}
	\begin{cases}
	E_{\min}(\theta(t),\psi(t)) &= (1-\sqrt{\alpha\beta})^2 \J(\theta(t)) + \frac 1 2 \left\Vert \aaa\theta(t) +\bb \psi(t) \right\Vert^2 \\
	E_{\max}(\theta(t),\psi(t)) &= (1+\sqrt{\alpha\beta})^2 \J(\theta(t)) + \frac 1 2 \left\Vert \aaa\theta(t) +\bb \psi(t) \right\Vert^2.
	\end{cases}
	\end{equation}

Then the following lemma applies.
	\begin{lemma}[Differentiation along DIN trajectories]\label{lem::dEdt}
		Let $(\theta,\psi)$ be a solution of \eqref{eq::contdin} with initial condition $(\theta_0,\psi_0)$. For a.e. $t>0$, $\theta$ and $\psi$ are differentiable at $t$, \eqref{eq::contdin} holds, $\vv \in D \J(\theta(t))$ and%
		\begin{align*}
		\frac{\diff E_{\min}}{\diff t}\left(\theta(t),\psi(t)\right) &= -\left\Vert 	\sqrt{\alpha}\dt(t) -\frac{1}{\sqrt{\beta}}\left(\ddp(t)-\dt(t)\right)\right\Vert^2 \\
	\frac{\diff E_{\max}}{\diff t}\left(\theta(t),\psi(t)\right) &= -\left\Vert 	\sqrt{\alpha}\dt(t) +\frac{1}{\sqrt{\beta}}\left(\ddp(t)-\dt(t)\right) \right\Vert^2
\end{align*}
	\end{lemma}

	\begin{proof}
		Define $\displaystyle E_\lambda(\theta,\psi) = \lambda {\mathcal J}(\theta) + \frac 1 2 \left\Vert \aaa\theta +\bb \psi \right\Vert^2 $. We aim to choose $ \lambda $ so that $ E_\lambda$ is a Lyapunov function. Because $\J$ is tame and locally Lipschitz continuous, using Lemma~\ref{lem:chainRuleSum} we know that for any absolutely continuous trajectory $ \theta:\R_+\to \R^P $ and for a.e. $ t>0 $,
		\begin{equation}
		\frac{\diff \J}{\diff t} (\theta(t)) = \langle \dot{\theta}(t), D \J(\theta(t)) \rangle = \langle \dot{\theta}(t) ,v(t) \rangle,\ \ \forall v(t)\in D \J(\theta(t)).
		\end{equation}
		Let $\theta$ and $\psi$ be solutions of (DIN). For a.e. $t\geq 0$, we can differentiate $ E_\lambda(\theta,\psi) $ to obtain
\begin{equation}\begin{split}
\frac{\diff E_\lambda}{\diff t}(\theta(t),\psi(t)) = & \lambda \langle \dot{\theta}(t) ,v(t) \rangle + \aaa \langle \dot{\theta}(t) ,\aaa\theta(t)+\bb\psi(t) \rangle \\
	& + \bb \langle \dot{\psi(t)} ,\aaa\theta(t)+\bb\psi(t)\rangle
	\end{split}
	\end{equation}
	for all $v(t)\in D \J(\theta(t))$. Using \eqref{eq::contdin}, we get $\frac{1}{\beta}(\dot\theta(t)-\dot \psi(t))\in D \J(\theta(t))$ and $-\ddp(t)=\aaa\theta(t)+\bb \psi(t)$ a.e. Choosing $v(t)=\frac{1}{\beta}(\dot\theta(t)-\dot \psi(t))$ yields:\begin{align*}
		  \frac{\diff E_\lambda}{\diff t}(\theta(t),\psi(t)) &= \lambda \left\langle \dot{\theta}(t) ,\vv \right\rangle - \aaa \left\langle \dot{\theta}(t) ,\ddp(t) \right\rangle - \bb \left\langle \ddp(t) ,\ddp(t) \right\rangle.
		\end{align*}
		Then, expressing everything as a function of $\dt$ and $\frac{1}{\beta}(\psi-\theta)$, one can show that a.e. on $\R_+$:
		\begin{align*}
		\frac{\diff E_\lambda}{\diff t}&(\theta,\psi)(t) =  - \alpha \Vert \dt(t)\Vert^2 -\beta \left\Vert  \vv \right\Vert^2 + \left(\lambda - \alpha\beta -1 \right)\langle \dot{\theta}(t) ,\vv \rangle\\
		&= -\left\Vert \sqrt{\alpha}\dt(t) + \frac{\alpha\beta +1 - \lambda }{2\sqrt{\alpha}}\vv\right\Vert^2 - \left (  \beta -\frac{(\alpha\beta +1 - \lambda)^2}{4\alpha}  \right)\left\Vert \vv \right\Vert^2.
		\end{align*}
		We aim to choose $ \lambda $ so that $ E_\lambda $ is decreasing that is $\left (  \beta -\frac{(\alpha\beta +1 - \lambda)^2}{4\alpha}  \right) > 0$. This holds whenever $ \lambda \in\, \left[(1-\sqrt{\alpha\beta})^2,(1+\sqrt{\alpha\beta})^2\right] $. We choose $ \lambda_{\min} =(1-\sqrt{\alpha\beta})^2 $, and $ \lambda_{\max} = (1+\sqrt{\alpha\beta})^2 $, for these two values we obtain for a.e. $ t>0 $ ,
		\begin{equation}\begin{cases}
		\dot E_{\lambda_{\min}}(\theta(t),\psi(t)) &= -\left\Vert \sqrt{\alpha}\dt(t) + \frac{1}{\sqrt{\beta}}\left(\dt(t)-\ddp(t)\right) \right\Vert^2 \vspace{0.2cm} \\
		\dot E_{\lambda_{\max}} (\theta(t),\psi(t))&= -\left\Vert \sqrt{\alpha}\dt(t) -  \frac{1}{\sqrt{\beta}}\left(\dt(t)-\ddp(t)\right)  \right\Vert^2
		\end{cases}
		\end{equation}
		Remark finally that by definition $E_{\min}=E_{\lambda_{\min}}$ and $E_{\max}=E_{\lambda_{\max}}$.
	\end{proof}
	 Recall that $ \displaystyle \ma  = \left\{ (\theta,\psi)\in \R^P\times \R^P \mid 0\in D J(\theta), \psi=(1-\alpha\beta) \theta) \right\}$ and define $ E = E_{\min} + E_{\max} $. By a direct integration argument, we obtain the following lemma.
	\begin{lemma}[$E$ is Lyapunov function for INNA with respect to $\ma$]\label{lem::Edec}
		 For all $ (\theta_0,\psi_0)\notin \ma$ and for any solution $(\theta,\psi)$ with initial condition $(\theta_0,\psi_0)$,
		\begin{equation}
		E(\theta(t),\psi(t))<E(\theta_0,\psi_0), \text{ for a.e. }t>0.
		\end{equation}
	\end{lemma}

\bigskip

\noindent
We are now in position to provide the desired proof.

\smallskip

\noindent
	{\bf Proof of Theorem \ref{th::thmDIN}} Lemmas~\ref{lem::dEdt} and~\ref{lem::Edec} state that
		$ E $ is a Lyapunov function for the set $ \ma $ and the dynamics \eqref{eq::contdin}. Let $\mathsf{C}=\{\theta\in \R^P\mid(\theta,\psi)\in \ma\}$ which is actually the set of $D$-critical points of $\J$. Using Lemma \ref{lem:sard} of Section \ref{sec:ResultsOnD}, $\J(\mathsf{C})$ is finite. Moreover, since $E(\theta,\psi)=2(1+\alpha\beta)\J(\theta) $ for all $(\theta,\psi)\in\ma$, $ E $ takes a finite number of values on $ \ma $, and in particular, $E(\ma)$ has empty interior.

		  Denote by $\ml$ the set of accumulation points of the sequences $((\theta_k, \psi_k))_{k \in \N}$ produced by \eqref{eq::discdin} starting at $(\theta_0,\psi_0)$ and $\ml_1$ its projection on $\R^P\times\{0\}$. We have the 3 following properties:
    \begin{itemize}
        \renewcommand\labelitemi{--}
	\item  By assumption, we have $\|(\theta_k, \psi_k)\| \leq M$ almost surely, for all $k \in \N$.\\
	\item  By local Lipschitz continuity $\partial \J_\mb(\theta)$ is uniformly bounded for $\|\theta\| \leq M$ and any $\mb \subset \{1,\ldots,N\}$, hence the centered noise $(\xi_k)_{k \in \N}$ is a uniformly bounded martingale difference sequence.\\
    \item  By Assumption \ref{ass:mainAssumption}, the sequence $ (\gamma_k)_{k \in \N} $ is chosen such that $ \gamma_{k} = o(\frac{1}{\log k}) $ (see Section~\ref{sec::comRes}).
    \end{itemize}
        \new{Then the sufficient conditions of Remark 1.5 of \cite{benaim2005stochastic} state that the discrete process $(\theta_k,\psi_k)_{k\in\N}$ asymptotically behaves like the solutions of \eqref{eq::contdin}. We can then combine Proposition 3.27 and Theorem 3.6 of \cite{benaim2005stochastic}, 
		to obtain that the limit set $\ml$ of the discrete process is contained in the set $\ma$ where the Lyapunov has vanishing derivatives. Thus, the set $\ml_1$ (the set of the first coordinates of all accumulation points) contains only $D$-critical points of $\J$. In addition, $E(\ml)$ is a singleton, and for all  $(\theta,\psi)\in\ma$, we have $E(\theta,\psi) = \J(\theta)$, so $\J(\ml_1) $ is also a singleton and the theorem follows.}

    \section{Towards Convergence Rates for INNA}\label{sec::rateofCV}

    In the previous section, connecting INNA to \eqref{eq::contdin} was one of the keys to prove the convergence of the discrete dynamics. Let us now focus on the continuous dynamical system \eqref{eq::contdinClarke} in the deterministic case where $\J$ and $\partial\J$ are not approximated anymore---we thus no longer use $D\J$ although this would be possible but would require more technical proofs. In this section and in this section only, we pertain to loss functions $\J$ that are real semi-algebraic (a particular case of tame functions).\footnote{We could extend the results of this section to more general objects including analytic functions on bounded sets. The semi-algebraicity assumption is made here for the sake of clarity.} Recall that a set is called semi-algebraic if it is a finite union of sets of the form,
    $$\{\theta\in \R^P\mid \zeta(\theta)=0,\zeta_i(\theta)<0\}$$
    where  $\zeta,\zeta_i$ are real polynomial functions. A function is called semi-algebraic if its graph is semi-algebraic. 
    
    We will characterize the convergence rate of the solutions of the continuous-time system \eqref{eq::contdinClarke} to critical points. Let us first introduce an essential mechanism to obtain such convergence rates: the Kurdyka-{\L}ojasiewicz (KL) property.

\subsection{The Non-smooth Kurdyka-{\L}ojasiewicz  Property for the Clarke Subdifferential}

The non-smooth Kurdyka-{\L}ojasiewicz (KL) property, as introduced in \citep{bolte2010}, is a measure of ``amenability to sharpness'' \new{(as illustrated at the end of Section~\ref{sec::rateIND})}. Here we provide a uniform version for the Clarke subdifferential of semi-algebraic functions as in \cite{bolte2007clarke} and \cite{bolte2014proximal}. In the sequel we denote by ``$\dist$'' any given distance on $\R^P$.
\begin{lemma}[Uniform Non-smooth KL Property for the Clarke Subdifferential]\label{lem::UKL}
     Let $\mathsf{K}$ be a nonempty compact set and let $L:\R^P\to \R$ be a semi-algebraic locally Lipschitz continuous function. Assume that $L$ is constant on $\mathsf{K}$, with value $L^\star$. Then there exist $\varepsilon>0$,  $\delta>0$, $a\in(0,1)$ and $\rho>0$ such that, for all
     \begin{equation*}
         v\in \left\{ v\in\R^P \mid \dist(v,\mathsf{K}) < \varepsilon \right\} \cap \left\{  v\in\R^P \mid L^\star < L(v) < L^\star+\delta  \right\},
     \end{equation*}
     it holds that,
     \begin{equation}\label{eq::SemiAlgKL}
         \rho(1-a)\left(L(v) - L(\bar{v})\right)^{-a}\dist\left( 0, \partial L(v)\right)  > 1.
     \end{equation}
 \end{lemma}

    \new{The proof directly follows from the general inequality provided in \cite{bolte2007clarke} or the local result of \cite{bolte2007clarke} with the compactness arguments of \citet[Lemma 6]{bolte2014proximal}.} In the sequel, we make an abuse of notation by writing $\Vert \partial\J(\cdot)\Vert \triangleq \dist(0,\partial \J(\cdot))$.
    To obtain a convergence rate we will use inequality \eqref{eq::SemiAlgKL} on the Lyapunov function $E$. But first we state a general result of convergence that is built around the KL property.

\subsection{A General Asymptotic Rate Result}
    We state a general theorem that leads to the existence of a convergence rate. This theorem will hold in particular for \eqref{eq::contdinClarke}. We start by stating the result.

     \begin{theorem}\label{th::genlemma}
        Let $X:[0,+\infty)\to \R^P$ be a bounded absolutely continuous trajectory and let $L:\R^P\to \R$ be a semi-algebraic locally Lipschitz continuous function. If there exists $c_1>0$ such that for a.e. $t>0$,
        \begin{equation}\label{it::2}\tag{i}
             \frac{\diff L}{\diff t} (X(t)) \leq -c_1 \Vert(\partial L)(X(t))\Vert^2,
        \end{equation}
        then $L(X(t))$ converges to a limit value $L^\star$ and,
        \[ \vert L(X(t)) - L^\star \vert = O\left(\frac{1}{t}\right). \]
        If in addition there exists $c_2>0$ such that for a.e. $t>0$,
        \begin{equation}\label{it::3}\tag{ii}
            c_2 \Vert \dot{X}(t)\Vert  \leq \Vert (\partial L)(X(t))\Vert,
        \end{equation}
    	then, $X$ converges to a critical point of $L$ with a rate of the form $O(1/t^b)$ with $b>0$.\footnote{In some cases we even have linear rates or finite convergence as detailed in the proof.}
    \end{theorem}
    
    \begin{proof}
    We first prove the convergence of $L(X(\cdot))$. Suppose that \eqref{it::2} holds. Since $X$ is bounded and $L$ is continuous, $L(X(\cdot))$ is bounded. Moreover, from \eqref{it::2}, $L(X(\cdot))$ is decreasing, so it converges to some value $L^\star$. To simplify suppose $L\geq 0$ and $L^\star=0$. Define,
    \begin{equation*}
			\mathsf{I} = \left\{ x\in\R^P \ \mid \ L(x)=0 \right\}.
	\end{equation*}
	Suppose first that there exists $s\geq 0$, such that $X(s)\in \mathsf{I}$. Since $L(X(\cdot))$ is decreasing with limit $0$, then for all $t\geq s$, $L(X(t)) = 0$ and the convergence rate holds true.
	
	Let us thus assume that for all $t\geq 0$, $L(X(t))>0$. The trajectory $X$ is bounded in $\R^P$, hence there exists a compact set $\mathsf{C}\subset\R^P$ such that $X(t)\in\mathsf{C}$ for all $t\geq 0$. Define $\mathsf{K}=\mathsf{I}\cap \mathsf{C}$. It is a compact set since $\mathsf{I}$ is closed (by continuity of $L$) and $\mathsf{C}$ is compact. Moreover, $L$ is constant on $\mathsf{K}$. As such by Lemma~\ref{lem::UKL}, there exist $\varepsilon>0,\ \delta>0$, $a\in(0,1)$ and a constant $\rho>0$ such that for all
		\[ v\in \left\{ v\in\R^P, \dist(v,\mathsf{K}) < \varepsilon \right\} \cap \left\{ 0 < L(v) < \delta  \right\},  \]
		it holds that
		\[\rho(1-a)\left(L(v)\right)^{-a}\dist\left( 0, \partial L(v)\right)  > 1.\]
    We have $L(X(t))\to 0$ so there exists $t_0\geq 0$ such that for all $t\geq t_0$, $0<L(X(t))<\delta$. Without loss of generality, we assume $t_0=0$. Similarly, we have $\dist(X(t),K) \to 0$, so we may assume that for all $t\geq 0$, $\dist\left(X(t),\mathsf{K}\right)<\varepsilon$. Thus, for all $t\geq 0$,
		\begin{equation*}
		\rho(1-a) L(X(t))^{-a}\Vert \partial L(X(t)) \Vert > 1.
		\end{equation*}
	Going back to assumption~\eqref{it::2}, for a.e. $t>0$, one has
		\[ \frac{\diff L}{\diff t} (X(t)) \leq -c_1 \Vert(\partial L)(X(t))\Vert^2, \]
		but the KL property implies that for a.e. $t>0$,
		$$- \Vert \partial L (X(t)) \Vert^2 < -\frac{1}{\rho^2(1-a)^2} L(X(t))^{2a}.$$ Therefore,
		\[\frac{\diff L}{\diff t} (X(t)) < -\frac{c_1}{\rho^2(1-a)^2} L(X(t))^{2a}. \]
		We consider two cases depending on the value of $a$. If $0<a\leq 1/2$, then for $t$ large, $L(X(t))<1$ so $-L(X(t))^{2a} < -L(X(t))$ and hence, $$\frac{\diff L}{\diff t} (X(t)) < -\frac{c_1}{\rho^2(1-a)^2} L(X(t)),$$ so we obtain a linear rate.
	When $1/2<a<1$, we have for a.e. $t>0$,
		\begin{equation}
		L(X(t))^{-2a} \frac{\diff }{\diff t}L(X(t)) =\frac{1}{1-2a}\frac{\diff}{\diff t} L(X(t))^{1-2a} < - \frac{c_1}{ (\rho^2 (1 - a)^2)},
		\end{equation}
		with $1-2a<0$. We can integrate from $0$ to $t>0$:
		\begin{equation*}
		 	L(X(t))^{1-2a}> \frac{(2a -1)c_1}{\rho^2(1-a)^2}t + L(X(0))^{1-2a}>\frac{(2a -1)c_1}{\rho^2(1-a)^2}t.
		\end{equation*}
		Since $\frac{1}{1-2a}<-1$, one obtains a convergence rate of the form $O\left(t^\frac{1}{1-2a}\right)$. In both cases the rate is at least $O\left(\frac{1}{t}\right)$.

	We assume now that both \eqref{it::2} and \eqref{it::3} holds and prove convergence of the trajectory with a convergence rate. 
    Let $t>s>0$, by the fundamental theorem of calculus and the triangular inequality,
    	\begin{equation}\label{eqpr::dotX}
    	\Vert X(t) - X(s)\Vert \leq \left\|  \int_s^{t} \dot{X}(\tau) \diff \tau \right\| \leq  \int_s^{t} \Vert\dot{X}(\tau)\Vert \diff \tau.
    	\end{equation}
    	We wish to bound $\Vert \dot{X}\Vert$ using $L$. Using the chain rule (Lemma~\ref{lem::chainrule} of Section~\ref{sec:ResultsOnD}), for a.e. $\tau>0$,
    	\begin{align}
    	\begin{split}
    	\frac{\diff }{\diff \tau}L(X(\tau))^{1-a} = (1-a)L(X(\tau))^{-a} \langle \dot{X}(\tau),(\partial L)(X(\tau))\rangle.
    	\end{split}
    	\end{align}
    	Then, from \eqref{it::2}, we deduce that for a.e. $\tau>0$,
    	\begin{align}
    	\begin{split}
    	\langle \dot{X}(\tau),(\partial L)(X(\tau))\rangle = \frac{\diff L}{\diff \tau}(X(\tau)) \leq -c_1 \Vert (\partial L)(X(\tau)) \Vert^2,
    	\end{split}
    	\end{align}
    	so
    	\begin{align}\label{eq::dtL}
    	\begin{split}
    	\frac{\diff }{\diff \tau}L(X(\tau))^{1-a} \leq -c_1(1-a)L(X(\tau))^{-a} \Vert (\partial L)(X(\tau)) \Vert^2.
    	\end{split}
    	\end{align}
    	The KL property \eqref{eq::SemiAlgKL} implies that for a.e. $\tau>0$,
    	\begin{equation}\label{eq::KLofL}
    	-(1-a)L(X(\tau))^{-a}\Vert (\partial L)(X(\tau))\Vert< -\frac{1}{\rho}.
    	\end{equation}
    	Putting this in \eqref{eq::dtL} and using assumption \eqref{it::3} we finally obtain
    	\begin{align}\label{eq::dtL2}
    	\begin{split}
    	\frac{\diff }{\diff t}L(X(\tau))^{1-a} < -\frac{c_1}{\rho} \Vert (\partial L)(X(\tau)) \Vert\leq -\frac{c_1c_2}{\rho} \Vert \dot{X}(\tau)\Vert.
    	\end{split}
    	\end{align}
    	We can use that in \eqref{eqpr::dotX},
    	\begin{equation}
    	\begin{split}\label{eq::boundX}
    	\Vert X(t) - X(s)\Vert &\leq  - \frac{\rho}{c_1 c_2}\int_s^{t} \frac{\diff }{\diff t}L(X(\tau))^{1-a} \diff \tau\\
    	&=  \frac{\rho}{c_1 c_2 } (L(X(s))^{1-a}-L(X(t))^{1-a}).
    	\end{split}
    	\end{equation}
    	Then, using the convergence rate we already proved for $L$, we deduce that the Cauchy criterion holds for $X$ inside the compact (hence complete) subset $\mathsf{C}\subset\R^P$ containing the trajectory. Thus, $X$ converges, and from (i) we have that $\lim\inf_{t \to +\infty} \|\partial L(X(t))\| = 0$ because $\partial L$ has closed graph. This shows that the limit is a critical point of $L$. Finally, taking the limit in \eqref{eq::boundX} and using the convergence rate of $L$ we obtain a rate for $X$ as well.
    \end{proof}

  \begin{remark}\label{rem::shift}
     {\rm   Theorem \ref{th::genlemma} takes the form of a general recipe to obtain a convergence rate since it may be applied in many cases, to curves or flows, provided that a convenient Lyapunov function is given. 
     Note also that it is sufficient for assumptions \eqref{it::2} and \eqref{it::3} to hold only after some time $t_0>0$ as in such case, one could simply do a time shift to use the theorem.} 
    \end{remark}

    \subsection{Application to INNA}\label{sec::rateIND}%
We now  apply Theorem~\ref{th::genlemma} to the deterministic continuous dynamical model of INNA \eqref{eq::contdinClarke}.
\begin{theorem}[Convergence rates]\label{cor::InnaRate}
        Suppose that $\J$ is semi-algebraic locally Lipschitz continuous and lower bounded. Then, any bounded trajectory $(\theta,\psi)$ that solves \eqref{eq::contdinClarke} converges to a point $(\bar{\theta},\bar{\psi}) \in \ma$,  with a convergence rate of the form $O\left(t^{-b}\right)$ with $b>0$. Moreover, $\J(\theta(t))$ converges to its limit $\bar \J$ with rate $\left\vert\J(\theta(t))-\bar \J\right\vert=O\left(\frac{1}{t}\right)$.
\end{theorem}

\begin{proof}
Let $(\theta,\psi)$ be a bounded solution of \eqref{eq::contdinClarke}. We would like to use Theorem~\ref{th::genlemma} with $X =(\theta,\psi)$, and a well-chosen function. In the proof of Theorem~\ref{th::thmDIN} we proved a descent property along the trajectory for the function $E(\theta,\psi) = 2(1+\alpha\beta)\J(\theta) + \left\Vert \aaa\theta +\bb \psi \right\Vert^2 $. This function is semi-algebraic, locally Lipschitz continuous, so it remains to prove that \eqref{it::2} and \eqref{it::3} hold for $E$ along $(\theta,\psi)$.

 For $t\geq 0$, denote $w(t) = (\alpha-\frac{1}{\beta})\theta(t) +\frac{1}{\beta}\psi(t)$, then according to Lemma~\ref{lem::dEdt} for a.e. $t>0$,
    \begin{align}\label{eq::valofDT}
        \begin{split}
        \frac{\diff E}{\diff t}(\theta(t),\psi(t)) &=- \Vert \sqrt{\alpha} \dot{\theta}(t) - \frac{1}{\sqrt{\beta}} \left( \dot{\psi}(t) - \dot{\theta}(t) \right)\Vert^2 - \Vert \sqrt{\alpha} \dot{\theta}(t) + \frac{1}{\sqrt{\beta}} \left( \dot{\psi}(t) - \dot{\theta}(t) \right)\Vert^2\\
        &=  -2\alpha \Vert \dot{\theta}(t)\Vert^2 - \frac{2}{\beta} \Vert \dot{\psi}(t) -\dot{\theta}(t) \Vert^2
        = -2\alpha \Vert \dot{\theta}(t) \Vert^2 - \frac{2}{\beta} \Vert \beta \partial \J(\theta(t)) \Vert^2\\
        &= -2\alpha \Vert -\beta\partial \J(\theta(t)) -w(t) \Vert^2 - 2\beta \Vert \partial \J(\theta(t)) \Vert^2.
        \end{split}
    \end{align}
    On the other hand, by standard results on the sum rule, we have for all $(\theta,\psi)\in\R^P\times\R^P$,
     \begin{equation}
         \partial E(\theta,\psi) = 2\begin{pmatrix}(1+\alpha\beta)\partial \J(\theta) + (\alpha-\frac{1}{\beta})\left( (\alpha-\frac{1}{\beta})\theta + \frac{1}{\beta}\psi\right)
         \\  \frac{1}{\beta} \left( (\alpha-\frac{1}{\beta})\theta + \frac{1}{\beta}\psi\right)
         \end{pmatrix},
     \end{equation}
    so for a.e. $t>0$,
    \begin{align}\label{eq::NormOfGrad}
        \frac{\left\| \partial E(\theta(t),\psi(t))\right\|^2}{4} = &\left\| (1+\alpha\beta)\partial \J(\theta(t)) + (\alpha-\frac{1}{\beta})w(t) \right\|^2 + \left\| \frac{1}{\beta} w(t) \right\|^2.
    \end{align}
    We wish to find $c_1>0$, such that $\frac12\frac{\diff E}{\diff t} + \frac{c_1}{4}\Vert \partial E\Vert^2 < 0$. This follows from the following claim.

        \textit{Claim}: let $r_1>0$, $r_2\in\R$, $r_3>0$, then there exist $C_1$ and $C_2$ two positive constants such that for any $a,b\in\R$,
        \begin{equation}\label{eq::analysisinequality}
                C_1(a^2+b^2) \leq (r_1 a +r_2 b)^2 + r_3 b^2 \leq C_2(a^2+b^2).
        \end{equation}
        Indeed, the function $Q:(a,b) \mapsto (r_1 a +r_2 b)^2 + r_3 b^2$ is a positive definite quadratic form, $C_1$ and $C_2$ can be taken to be two eigenvalues of the positive definite matrix which represents $Q$. Hence  \eqref{eq::analysisinequality} holds for all $a$ and $b$.

    Applying the previous claim to \eqref{eq::NormOfGrad} and \eqref{eq::valofDT} leads to the existence of $c_1>0$ such that for a.e. $t>0$,\begin{equation*}
        \frac{\diff E}{\diff t}(\theta(t),\psi(t))\leq - c_1 \Vert \partial E (\theta(t),\psi(t))\Vert^2,
    \end{equation*} so assumption \eqref{it::2} holds for INNA.

    It now remains to show that \eqref{it::3} of Theorem~\ref{th::genlemma} holds i.e., that there exists $c_2>0$ such that for $(\theta,\psi)$ solution of \eqref{eq::contdinClarke} and for a.e. $t>0$, $\Vert \partial E(\theta(t),\psi(t))\Vert^2\geq c_2 \left( \Vert \dot{\theta}(t)\Vert^2 + \Vert \dot{\psi}(t)\Vert^2 \right)$.
    Using \eqref{eq::contdinClarke} and \eqref{eq::NormOfGrad} we obtain:
    \begin{align}\label{eq::assum3}
        \begin{split}
            \frac{\Vert \partial E(\theta(t),\psi(t))\Vert^2}{4} =\left\| \frac{1}{\beta}(1+\alpha\beta)\dot{\theta}(t) + \left[(\alpha-\frac{1}{\beta}) -\frac{1}{\beta}(1+\alpha\beta)\right]\dot{\psi}(t)\right\|^2 + \frac{1}{\beta^2}\Vert\dot{\psi}(t)\Vert^2 
        \end{split},
    \end{align}
    and applying the claim \eqref{eq::analysisinequality} again to \eqref{eq::assum3} one can show that there exist $c_2>0$, such that for a.e. $t>0$, \[\Vert \partial E(\theta(t),\psi(t))\Vert^2\geq c_2 \left( \Vert \dot{\theta}(t)\Vert^2 + \Vert \dot{\psi}(t)\Vert^2 \right).\] So assumption \eqref{it::3} holds for \eqref{eq::contdinClarke}. To conclude, we can apply Theorem~\ref{th::genlemma} to \eqref{eq::contdinClarke} and the proof is complete.
    \end{proof}

    \begin{remark}\label{rem::rateofIND}
       \rm{(a)  \new{Since the discrete algorithm INNA asymptotically resembles its continuous-time version (see the proof of Theorem~\ref{th::thmDIN}), the results above suggest that similar behaviors and rates could be hoped for INNA itself. Yet, these results remain difficult to obtain in the case of DL, in particular in the mini-batch setting because of the noise $(\xi_k)_{k\in\N}$.}\\
       (b) The proof above is significantly simpler when $\alpha\beta>1$ since \citet{alvarez2002second} proved that in this case, \eqref{eq::contdinClarke} is equivalent to a gradient system, thus assumptions \eqref{it::2} and \eqref{it::3} of Theorem~\ref{th::genlemma} instantly hold.}\\
       \rm{ (c)  Theorems \ref{th::genlemma} and \ref{cor::InnaRate} can be adapted to the case where the Clarke subdifferential is replaced by $D\J$, but we do not state it here for the sake of simplicity.}\\
       \rm{ (d) \new{Theorems~\ref{th::genlemma} and~\ref{cor::InnaRate} are actually valid by assuming that $\J$ belongs to a polynomially bounded o-minimal structure. One of the most common instance of such structures is the one given by globally subanalytic sets (as illustrated in a example below). We refer to \cite{bolte2007lojasiewicz} for a definition and further references.}}
    \end{remark}
    
    \new{Let us now comment the results of Theorem~\ref{cor::InnaRate}. First, we restrained here to semi-algebraic loss functions $\J$, which are a subset of tame loss functions. Most networks, activation functions and dissimilarity measures mentioned in Section~\ref{sec::favstruct} fall into this category. Nonetheless, the loss functions of the DL experiments of Section~\ref{sec::trainDL} are not semi-algebraic.  Indeed, the dissimilarity measure $l$ used is the cross-entropy: $l(f(x_n,\theta),y_n)= - \sum_{d=1}^D \indic_{[y_{n}]_d=1}\log( [f(x_n,\theta)]_d)$. Such a function cannot be described by polynomials and presents a singularity whenever $[f(x_n,\theta)]_d=0$. Fortunately, due to the numerical precision but also to the ``soft-max'' functions often used in classification experiments, the outputs of the network $f$, for inputs restricted to a compact set, have values in $[\varepsilon,1]$ for some small $\varepsilon>0$. Therefore, the singularity at $0$ is harmless and the cross-entropy acts as a globally subanalytic function. As a consequence the non-smooth {\L}ojasiewicz inequality holds, and the theorems  apply (see also numerical experiments).
        
         The rate of convergence of the trajectory in Theorem~\ref{cor::InnaRate} is non-explicit in the sense that the exponent $b>0$  is unknown in general. In the light of the proof of Theorem~\ref{th::genlemma}, this exponent depends on the KL exponent $a$ of the Lyapunov function, which is itself hard to determine in practice. However, the intuition is that small exponents $a$ may yield faster convergence rates (indeed, when $a\in(0,1/2)$ we actually have a linear rate). As an example, for the function: $t\in\R\mapsto \vert t \vert^c$ with $c>1$, the exponent at $0$ is $a=1-\frac{1}{c}$ and thus, the closer 
         $c$ is to $1$, the smaller $a$ is, and the faster the convergence becomes.
    }

    \section{Experiments}
	\label{sec:numerics}

In this section we first discuss the role and influence of the hyper-parameters of INNA using the 2D example given in Figure~\ref{fig::rosenbrockexp}. We then compare INNA with SGD, ADAGRAD and ADAM on deep learning problems for image recognition.

	\subsection{Understanding the Role of the Hyper-parameters of INNA}\label{sec::HPmeaning}
 Both  hyper-parameters $\alpha$ and $\beta$ can be seen as damping coefficients from the viewpoint of mechanics as discussed by \cite{alvarez2002second} and sketched in the introduction. Recall the second-order time continuous dynamics which served as a model to the design of INNA:
	   \begin{equation*}
    	  \ddot{\theta}(t)+\alpha\,\dt(t) + \beta\, \nabla^2 \J(\theta(t))\dt(t) + \nabla \J(\theta(t))=0.
	   \end{equation*}
	This differential equation was inspired by Newton's second law of dynamics asserting that the acceleration of a material point   coincides with the sum of forces applied to the particle. As recalled in the introduction three forces are at stake: the gravity and two friction terms. The parameter $\alpha$ calibrates the {\em viscous damping} intensity as in the Heavy Ball friction method of \cite{polyak1964some}. It acts as a dissipation term but it can also be seen as a proximity parameter of the system with the usual gradient descent: the higher $\alpha$ is, the more DIN behaves like a pure gradient descent.\footnote{This is easier to see when one rescales $\J$ by $\alpha$.} On the other hand the parameter $\beta$ can be seen as a \textit{Newton damping} which takes into account the geometry of the landscape to brake or accelerate the dynamics in an adaptive anisotropic fashion, see \cite{felipe,alvarez2002second} for further insights.

	We now turn our attention to INNA, and illustrate the versatility of the hyper-parameters $\alpha$ and   $\beta$ in this case. We proceed on a 2D visual non-smooth ill-conditioned example \`a la Rosenbrock, see Figure~\ref{fig::rosenbrockexp}. For this example, we aim to find the minimum of the function $\J(\theta_1,\theta_2)= 100(\theta_2-\vert \theta_1 \vert)^2 +\vert 1-\theta_1\vert $. This function has a {\em V}-shaped valley, and a unique critical point at $(1,1)$ which is also the global minimum. Starting from the point $(-1,1.5)$ (the black cross), we apply INNA with constant steps $\gamma_k=10^{-4}$. Figure \ref{fig::rosenbrockexp} shows that when $\beta$ is too small, the trajectory presents many transverse oscillations as well as longitudinal ones close to the critical point (subplot a). Then, increasing $\beta$ significantly reduces  transverse oscillations (subplot b). Finally, the longitudinal oscillations are reduced by choosing a higher $\alpha$ (subplot c). In addition, these behaviors are also reflected in the values of the objective function (subplot d).
	The orange curve (first setting) presents large oscillations. Moreover, looking at the red curve, corresponding to plot (c), there is a short period between $20,000$ and $60,000$ iterations when the decrease is slower than for the other values of $\alpha$ and $\beta$, but still it presents fewer oscillations. In the longer term, the third choice ($\alpha=1.3$, $\beta=0.1$) provides remarkably good performance.%
	
    \new{The choice of these hyper-parameters may come with rates of convergence for convex and strongly convex smooth functions \citep{attouch2019first}. Following this work, one may also consider to make $\alpha$ and $\beta$ vary in time (for example like the famous Nesterov damping coefficient $\frac{\alpha}{t}$). In our DL experiments we will however keep these parameters constant so that our theorems still hold. Yet, different behaviors depending on $(\alpha,\beta)$ can also be observed for DL problems as illustrated on Figure~\ref{fig::Multiparam} and described next.} Although we did not evidence some universal method to choose $(\alpha,\beta)$, we used mechanical intuitions to tune these parameters. The coefficient $\alpha$ induces viscous damping, thus one may try to reduce it when convergence appears to be slow. On the other hand, one may want to increase $\beta$ when large oscillations are observed. Yet, since $\beta$ affects directly  the subgradient effect in \eqref{eq::discdin}, taking $\beta$ too large may jeopardize the numerical stability of the algorithm. It would be interesting to relate the roles of these coefficients with more structural aspects of the dynamics. Indeed, as mentioned in Remark~\ref{rem::rateofIND}-b, when $\alpha\beta\geq 1$, \eqref{eq::contdin} can be shown to be a gradient system. On the other hand, when $\alpha\beta<1$ the dynamics is of a different type, we refer to \citet{alvarez2002second} for further comments on this matter.

	\subsection{Training a DNN with INNA}\label{sec::trainDL}
	    \begin{figure}[t]
        \begin{center}
            \begin{minipage}[c]{0.33\textwidth}
                \centering \footnotesize (a) CIFAR-10
            \end{minipage}%
            \begin{minipage}[c]{0.33\textwidth}
                \centering \footnotesize (b) CIFAR-100
            \end{minipage}%
            \begin{minipage}[c]{0.33\textwidth}
                \centering \footnotesize (c) MNIST
            \end{minipage}%
            
            \begin{minipage}[c]{0.33\textwidth}
                \centering
                \includegraphics[clip, trim=0.40cm 0.5cm 0.45cm 0.3cm, width=0.95\linewidth]{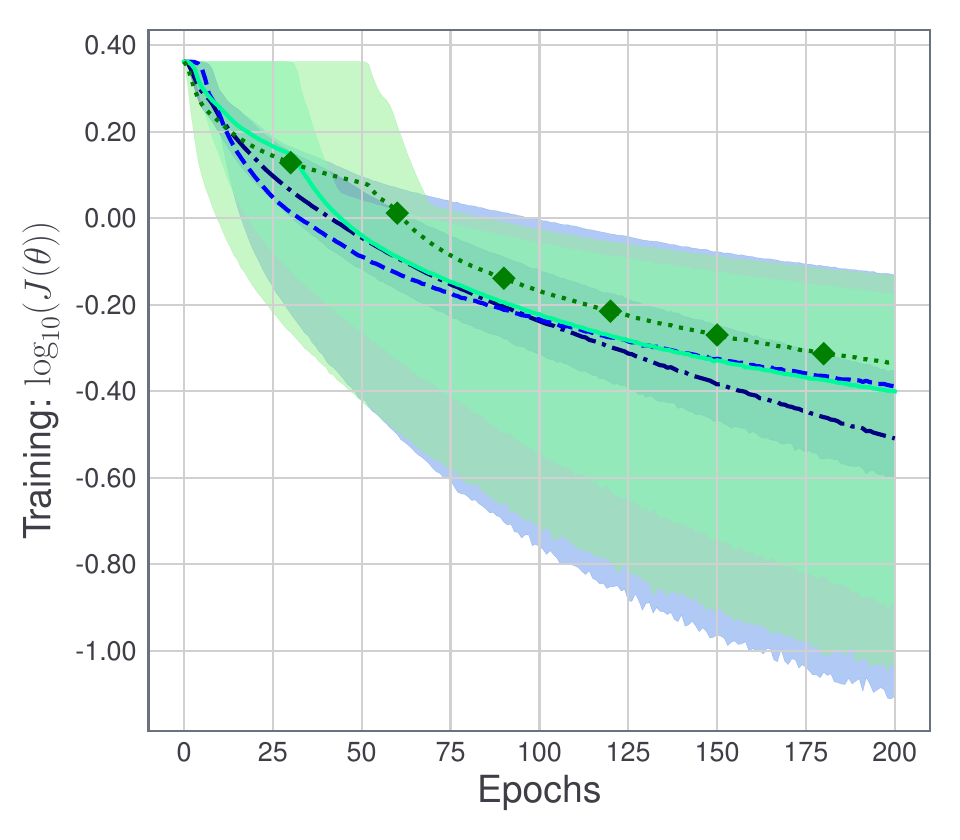}
            \end{minipage}%
            \begin{minipage}[c]{0.33\textwidth}
                \centering
                \includegraphics[clip, trim=0.40cm 0.5cm 0.35cm 0.3cm, width=0.95\linewidth]{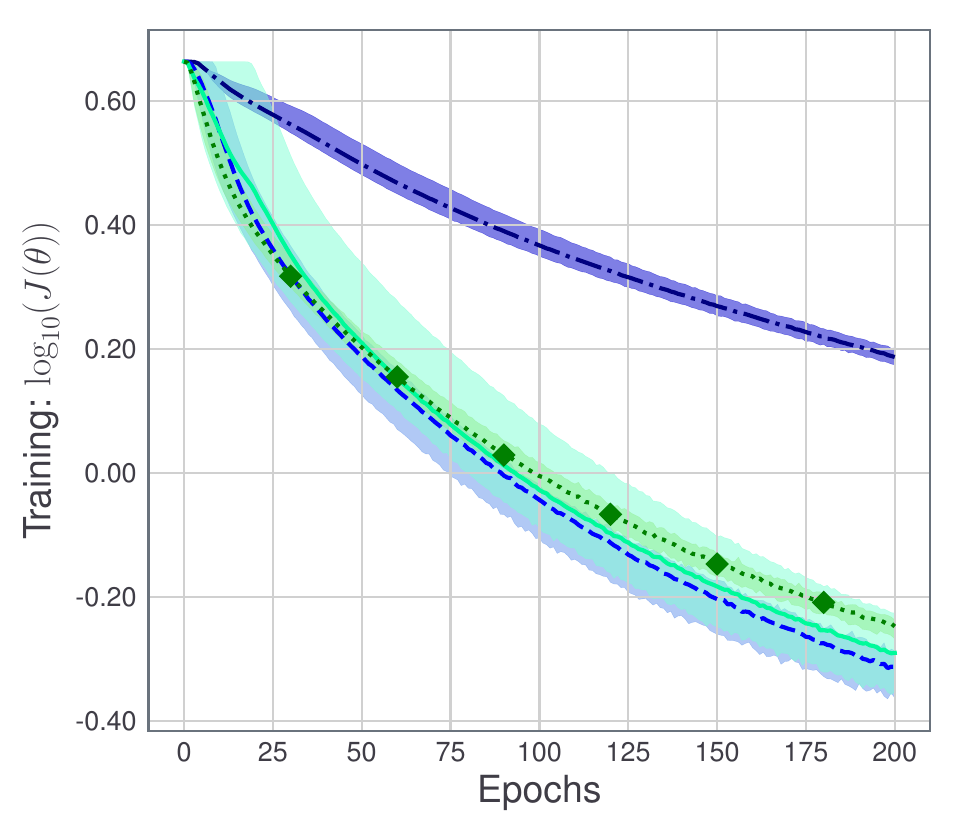}
            \end{minipage}%
            \begin{minipage}[c]{0.33\textwidth}
                \centering\includegraphics[clip, trim=0.40cm 0.5cm 0.35cm 0.3cm, width=0.95\linewidth]{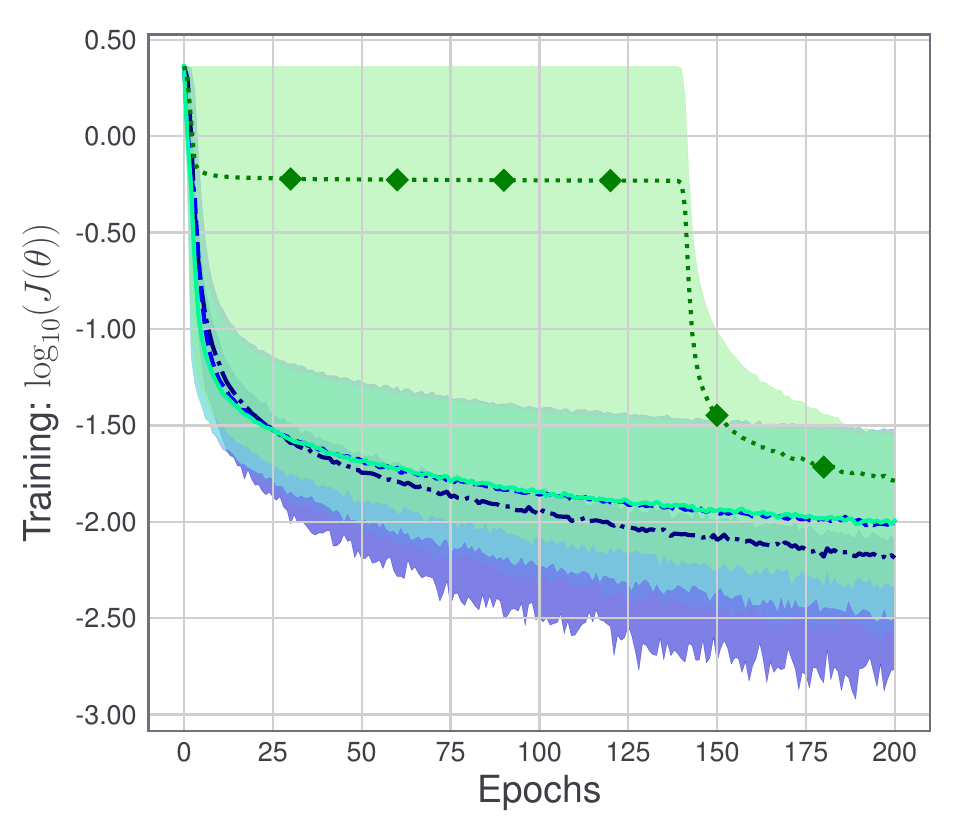}
            \end{minipage}%
            
            \begin{minipage}[c]{0.33\textwidth}\centering\includegraphics[clip, trim=0.2cm 0.5cm 0.35cm 0.3cm, width=0.95\linewidth]{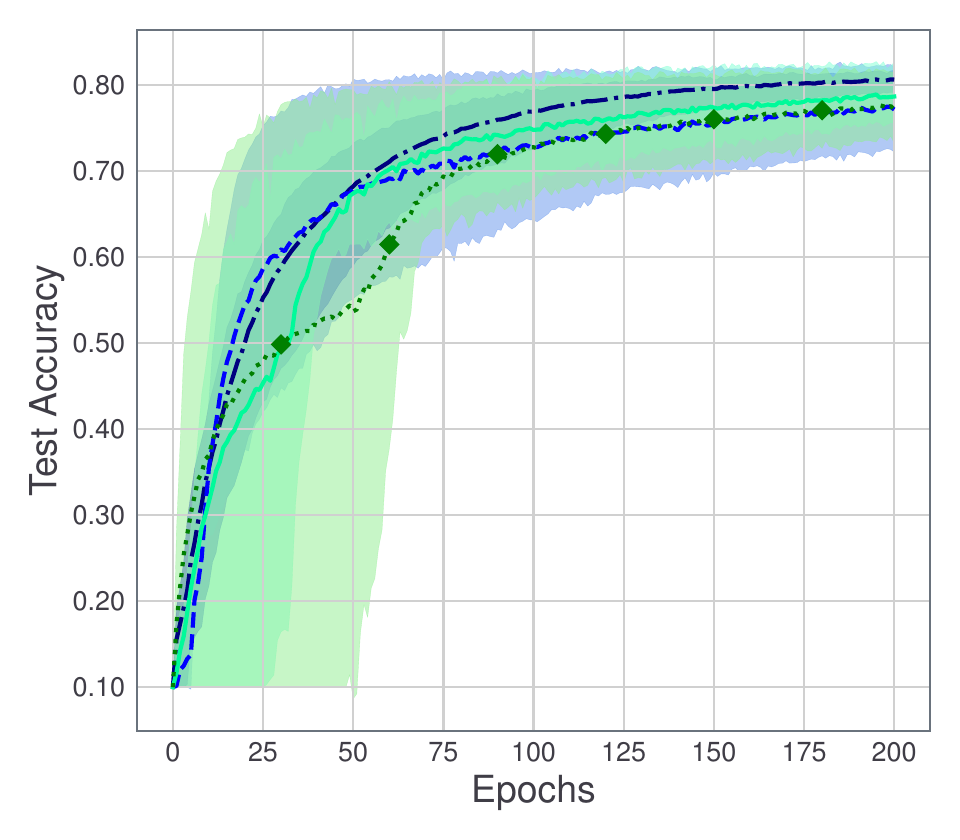}
            \end{minipage}%
            \begin{minipage}[c]{0.33\textwidth}\centering\includegraphics[clip, trim=0.5cm 0.5cm 0.3cm 0.5cm, width=0.95\linewidth]{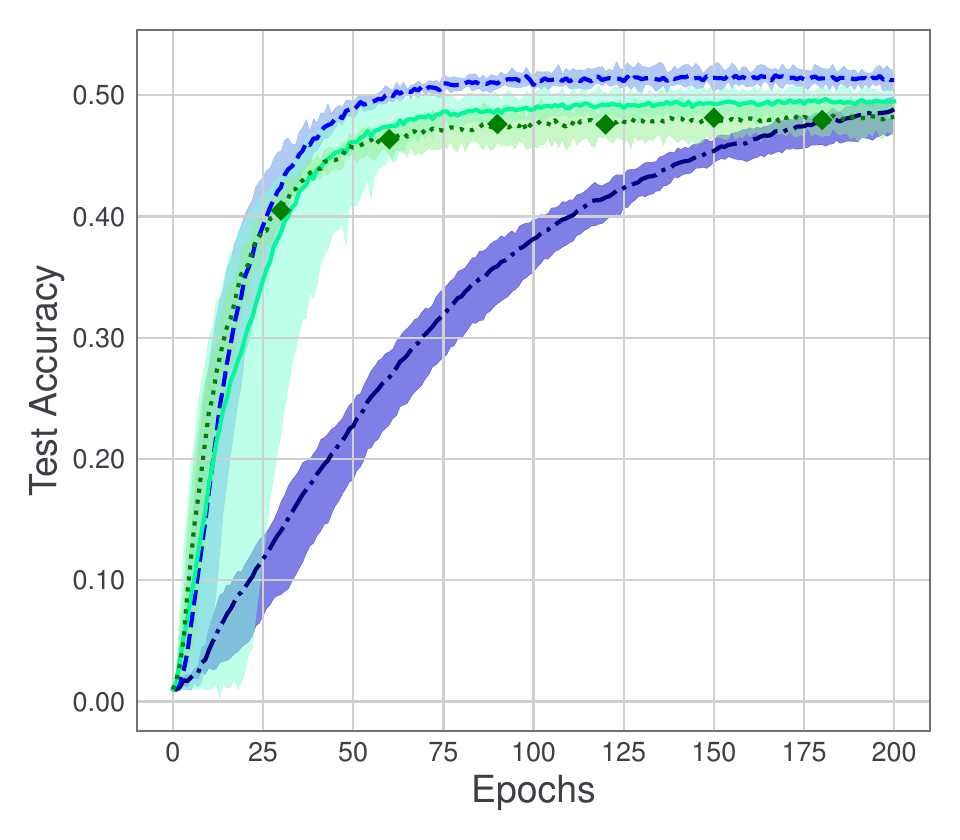}
            \end{minipage}%
            \begin{minipage}[c]{0.33\textwidth}\centering\includegraphics[clip, trim=0.5cm 0.5cm 0.3cm 0.5cm, width=0.95\linewidth]{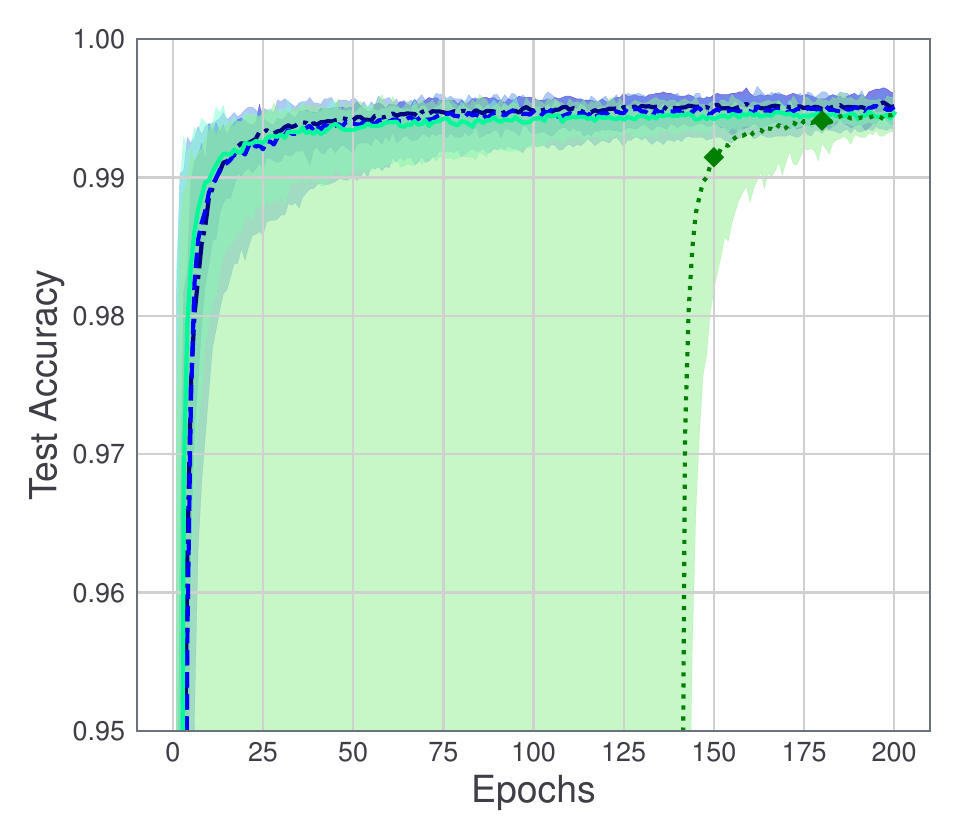}
            \end{minipage}%
            
            \begin{minipage}[c]{0.8\linewidth}
                \centering\includegraphics[width=0.45\linewidth]{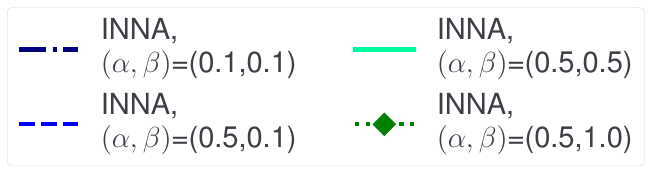}
            \end{minipage}%
            \caption{Analysis of the sensibility of INNA to the choice of $\alpha$ and $\beta$ using NiN for three different image classification problems. Top:  logarithm of the loss function $\J(\theta)$ during the training. Bottom: classification accuracy on the test set.}\label{fig::Multiparam}
        \end{center}
    \end{figure}

	Before comparing INNA to state-of-the-art algorithms in DL, we first describe the methodology that we followed.

	\subsubsection{Methodology}
	\begin{itemize}
	    \renewcommand\labelitemi{--}
	    \item We train a DNN for classification using the three most common image data sets (MNIST, CIFAR-10, CIFAR-100) \citep{lecun1998gradient,krizhevsky2009learning}. These data sets are composed of $60,000$ small images associated with a label (numbers, objects, animals, etc.). We split the data sets into $50,000$ images for training and $10,000$ for testing.
	    \item 	Regarding the network, we use a slightly modified version of the popular Network in Network (NiN) \citep{lin2013network}. It is a reasonably large convolutional network with $P\sim10^6$ parameters to optimize. We use ReLU activation functions.
	    \item The dissimilarity measure $l$ that is used in the empirical loss $\J$ given by \eqref{eq::loss} is set to the cross-entropy.
	    The loss function $\J$ is optimized with respect to $\theta$ (the weights of the DNN) on the training data. The classification accuracy of the trained DNN is measured using the test data of $10,000$ images. Measuring the accuracy boils down to counting how many of the $10,000$ were correctly classified (in percentage).
	    \item 	Based on the results of Section \ref{sec::HPmeaning}, we run INNA for four different values of $(\alpha,\beta)$: \[(\alpha,\beta)\in \left\{(0.1,0.1),(0.5,0.1),(0.5,0.5),(0.5,1)\right\}.\]
    	Given an initialization of the weights $\theta_0$, we 
    	initialize $\psi_0$ such that the initial
    	velocity is in the direction of $-\nabla 
    	\J(\theta_0)$. More precisely, we use $ \psi_0 = 
    	(1-\alpha\beta)\theta_0 
    	-(\beta^2-\beta)\nabla \J(\theta_0)$.
    	\item We compare our algorithm INNA with the classical SGD algorithm and the popular ADAGRAD \citep{duchi2011adaptive} and ADAM \citep{kingma2014adam} algorithms. At each iteration $k$, we compute the approximation of $\partial \J(\theta)$ on a subset $\mb_k\subset\left\{1,\ldots,50,000\right\}$ of size $32$. The algorithms are initialized with the same random weights (drawn from a normal distribution). Five random initializations are considered for each experiment.
    	\item Regarding the selection of step-sizes, ADAGRAD and ADAM both use an adaptive procedure based on past gradients, see \cite{duchi2011adaptive,kingma2014adam}. For the other two algorithms (INNA and SGD), we use the classical step-size schedule $\gamma_k=\frac{\gamma_0}{\sqrt{k+1}}$, which meets Assumption~\ref{ass:mainAssumption}.  For all four algorithms, choosing the right initial step length $\gamma_0$ is often critical in terms of efficiency. \new{We choose this $\gamma_0$ using a grid-search: for each algorithm we select the initial step-size that most decreases the training error $\J$ after fifteen epochs (one epoch consisting in a complete pass over the data). The test data is not used to choose the initial step-size nor other hyper-parameters.} Note that we could use more flexible step-size schedules but chose a standard schedule for simplicity. Other decay schemes are considered in Figure~\ref{fig::decay}.

	\end{itemize}
     For these experiments, we used \texttt{Keras 2.2.4} \citep{chollet2015} with \texttt{Tensorflow 1.13.1} \citep{abadi2016tensorflow} as backend. The INNA algorithm is available in Pytorch, Keras and Tensorflow: \url{https://github.com/camcastera/Inna-for-DeepLearning/} \citep{castera2019github}.

	\subsubsection{Results}\label{sec:numreslenet}
    
    Figure~\ref{fig::Multiparam} displays the training loss $\J$ and test accuracy with respect to epochs for INNA in its four hyper-parameter configurations considered and for the three data sets considered. Figure~\ref{fig::vsothalgo} displays the performance of INNA \new{with the hyper-parameter configuration that led to the smallest average training error in Figure~\ref{fig::Multiparam}}, with comparison to SGD, ADAGRAD and ADAM. In these two figures (and also in subsequent Figure~\ref{fig::decay}), solid lines represent mean values and pale surfaces represent the best and worst runs in terms of training loss and validation accuracy over five random initializations.

 		Figure~\ref{fig::Multiparam} suggests that the tuning of the hyper-parameters $\alpha$ and $\beta$ is not crucial to obtain satisfactory results both for training and testing. It mostly affects the training speed. Thus, INNA looks quite stable with respect to these hyper-parameters. Setting $(\alpha,\beta)=(0.5,0.1)$ appears to be a good default choice when necessary. Nevertheless, tuning these hyper-parameters is of course advised to get the most from INNA.
 		
 		\begin{figure}[t]
     	\begin{center}
                \begin{minipage}[c]{0.33\textwidth}
 				\centering \footnotesize (a) CIFAR-10
 				\end{minipage}%
 				\begin{minipage}[c]{0.33\textwidth}
 				\centering \footnotesize (b) CIFAR-100
 				\end{minipage}%
 				\begin{minipage}[c]{0.33\textwidth}
 				\centering \footnotesize (c) MNIST
 				\end{minipage}%

 				\begin{minipage}[c]{0.33\textwidth}
                \centering
 						\includegraphics[clip, trim=0.40cm 0.5cm 0.45cm 0.3cm, width=0.95\linewidth]{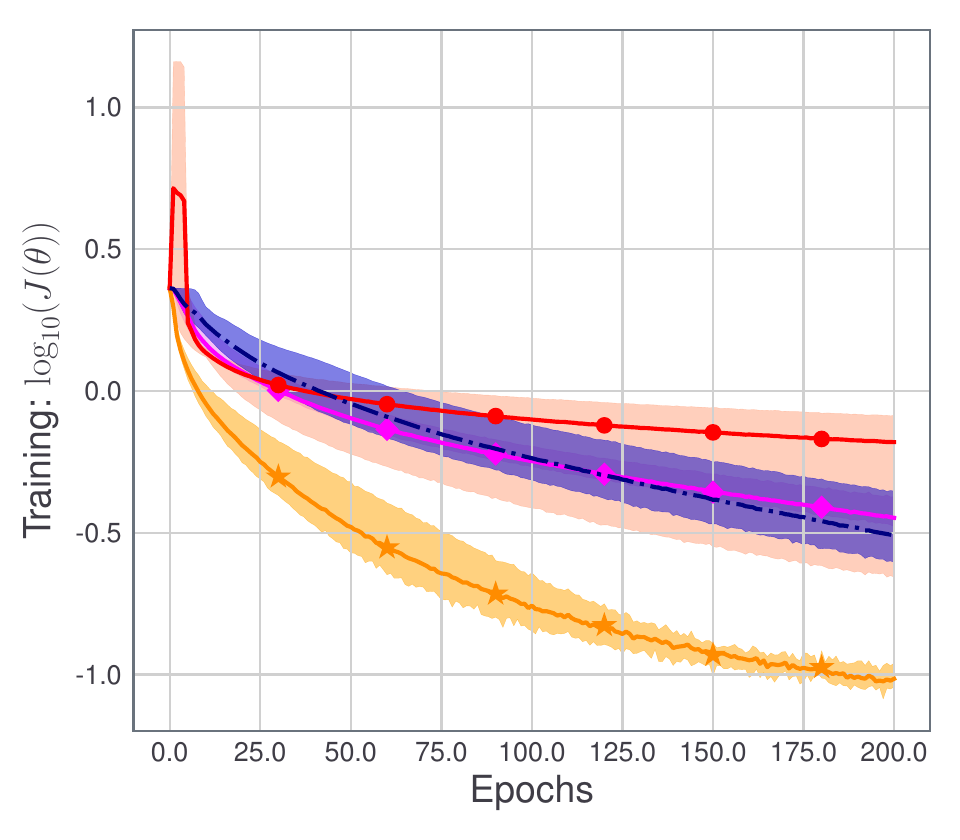}
 				\end{minipage}%
 				\begin{minipage}[c]{0.33\textwidth}
 				\centering
 						\includegraphics[clip, trim=0.40cm 0.5cm 0.35cm 0.3cm, width=0.95\linewidth]{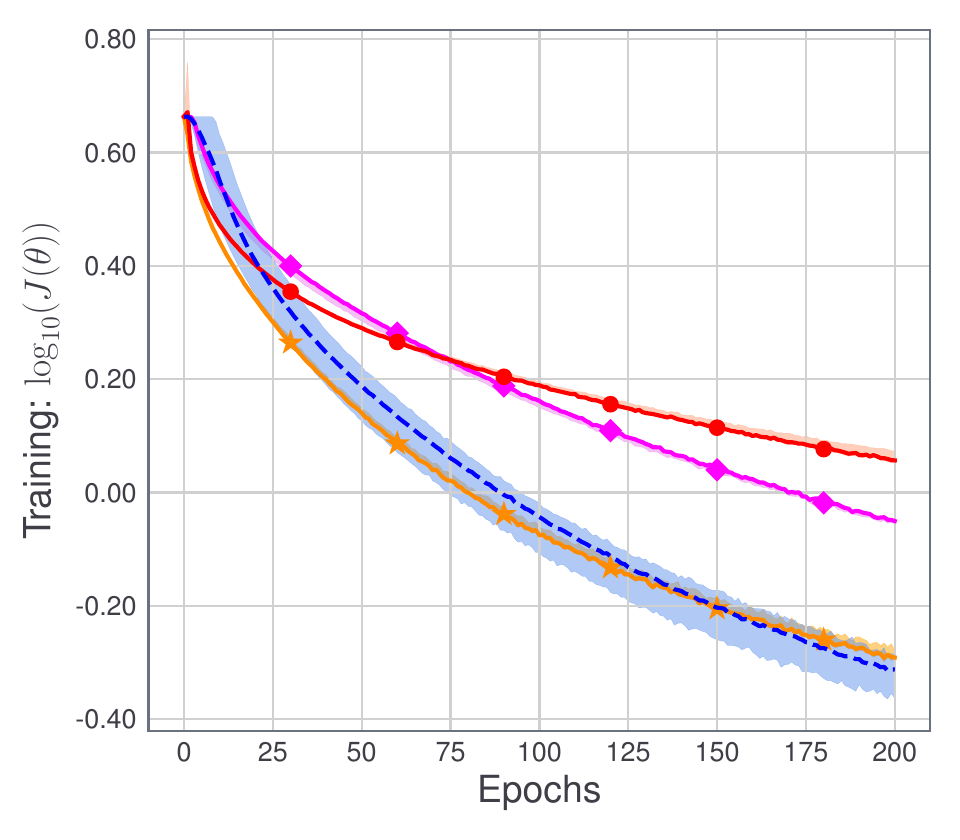}
 				\end{minipage}%
 				\begin{minipage}[c]{0.33\textwidth}
						\centering\includegraphics[clip, trim=0.40cm 0.5cm 0.35cm 0.3cm, width=0.95\linewidth]{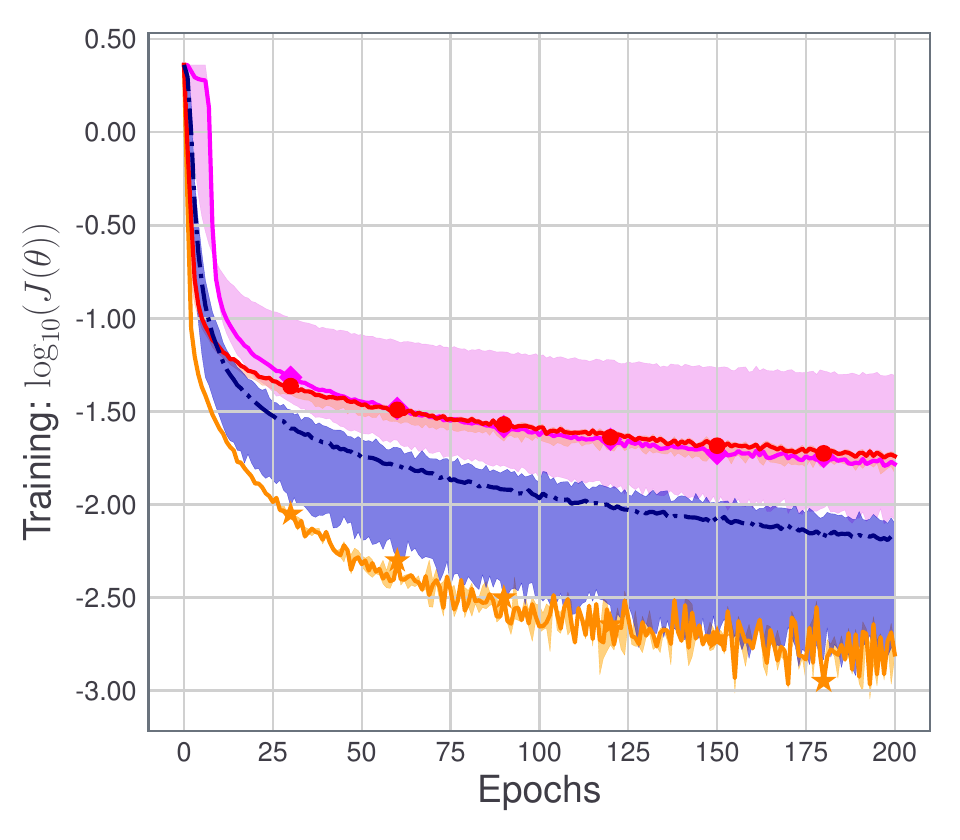}
				\end{minipage}%

 				\begin{minipage}[c]{0.33\textwidth}\centering\includegraphics[clip, trim=0.2cm 0.5cm 0.35cm 0.3cm, width=0.95\linewidth]{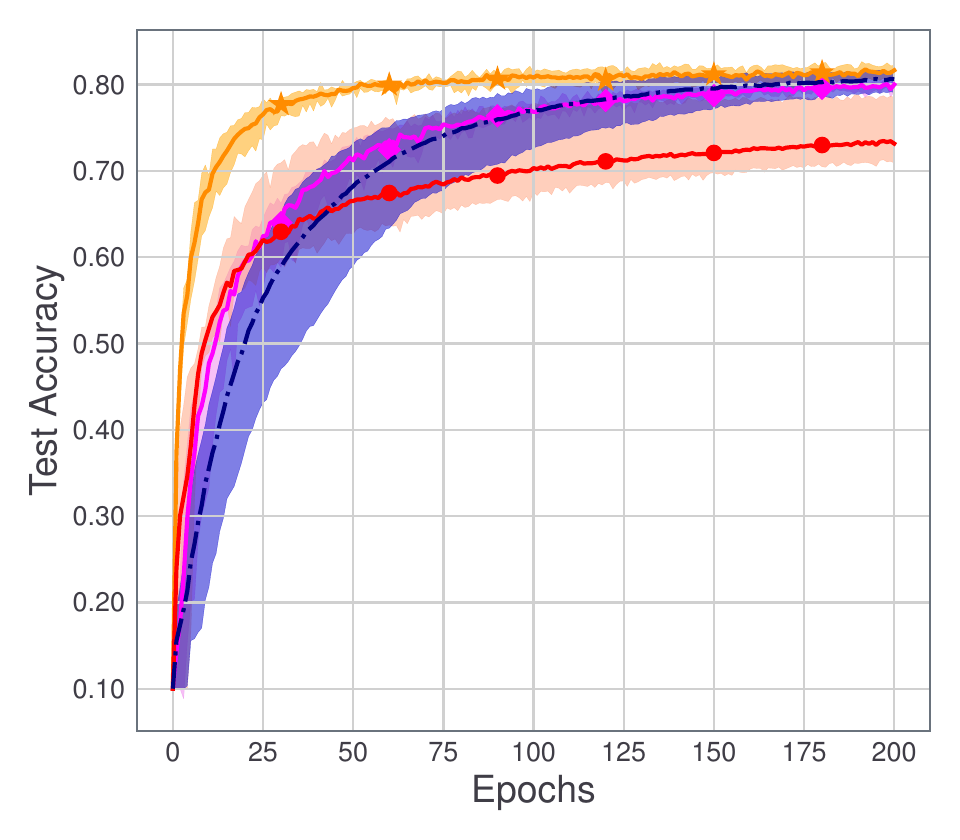}
 				\end{minipage}%
 				\begin{minipage}[c]{0.33\textwidth}\centering\includegraphics[clip, trim=0.5cm 0.5cm 0.3cm 0.5cm, width=0.95\linewidth]{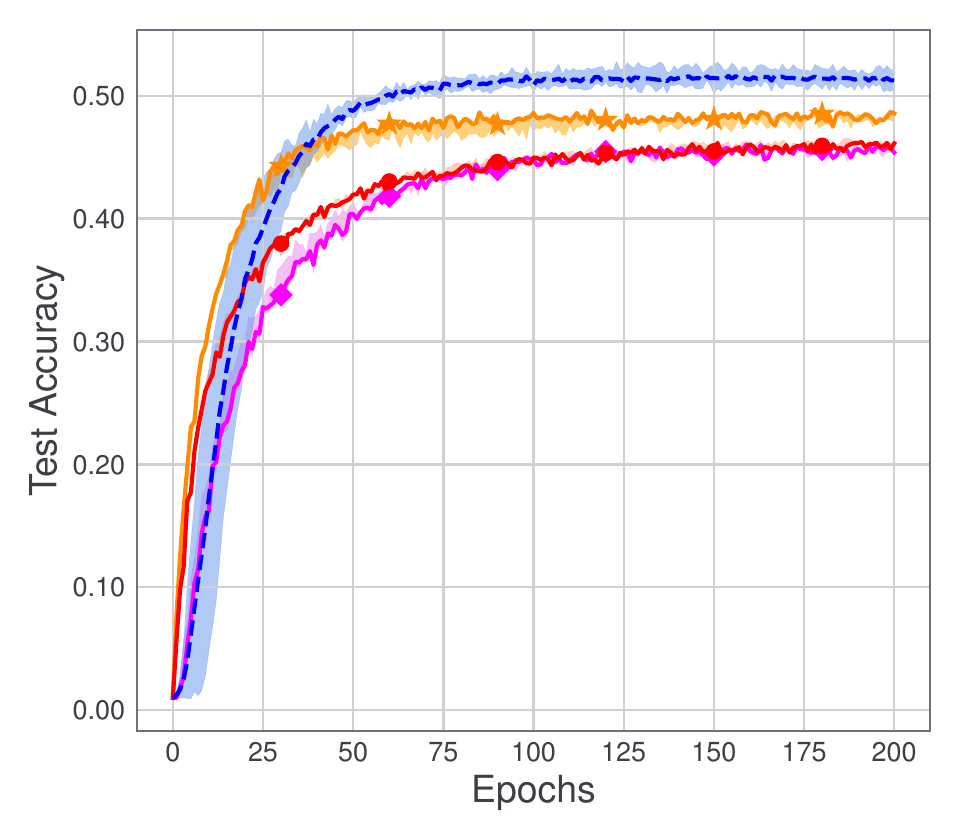}
				\end{minipage}%
				\begin{minipage}[c]{0.33\textwidth}\centering\includegraphics[clip, trim=0.5cm 0.5cm 0.3cm 0.5cm, width=0.95\linewidth]{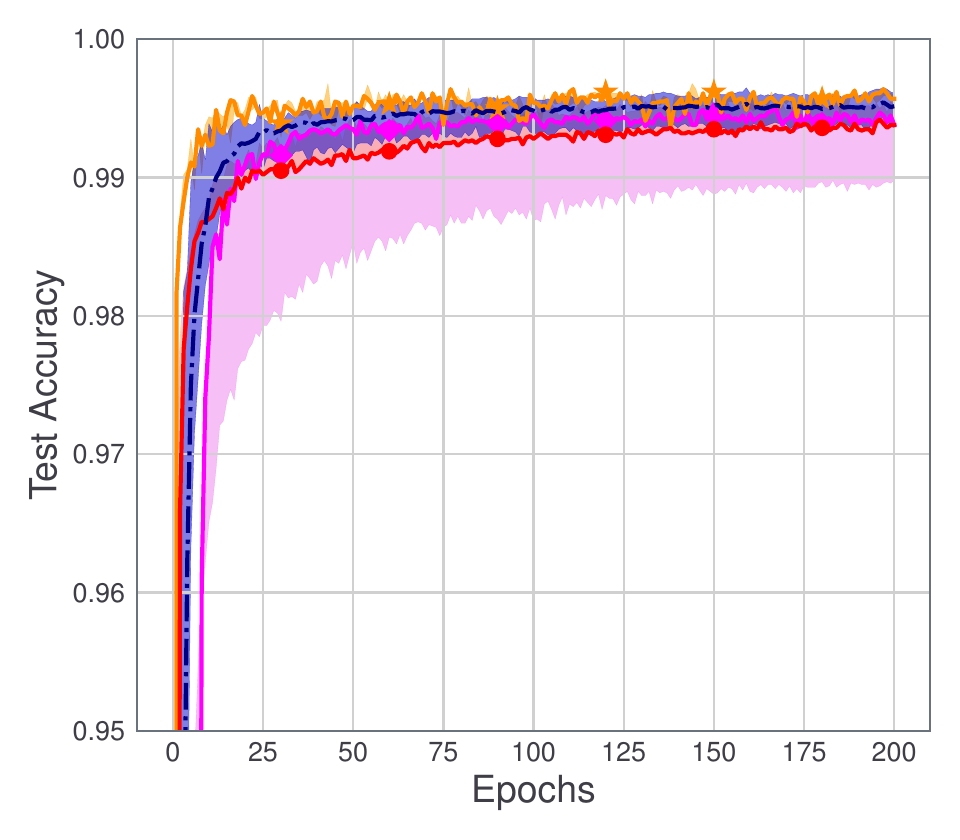}
 				\end{minipage}%

 			 	\begin{minipage}[c]{0.8\linewidth}
 			 	\centering\includegraphics[width=0.45\linewidth]{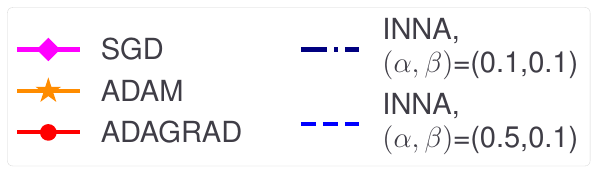}
 			 	\end{minipage}%
 			 	\caption{Comparison of INNA with state-of-the-art algorithms SGD, ADAM and ADAGRAD. Top: logarithm of the loss function $\J(\theta)$ during the training. Bottom: classification accuracy on the test set.
 			 	}
 			 	\label{fig::vsothalgo}
 			 	\end{center}
 		\end{figure}
 	
Figure~\ref{fig::vsothalgo} shows that best performing methods achieve state-of-the art accuracy using NiN and represent what can be achieved with a moderately large network and coarse grid-search tuning of the initial step-size. In our comparison, INNA and ADAM outperform SGD and ADAGRAD for training. While ADAM seems to be faster in the early training phase, INNA achieves the best accuracy almost every time especially on CIFAR-100 (Figure~\ref{fig::vsothalgo}(b)). Thus, INNA appears to be competitive in comparison to other algorithms with the advantage of having solid theoretical foundations and a simple step-size rule as compared to ADAM and ADAGRAD.

 		\begin{figure}[t]
     	\begin{center}
                \begin{minipage}[c]{0.33\textwidth}
 				\centering \footnotesize (a) CIFAR-10
 				\end{minipage}%
 				\begin{minipage}[c]{0.33\textwidth}
 				\centering \footnotesize (b) CIFAR-100
 				\end{minipage}%
                \begin{minipage}[c]{0.33\textwidth}
 				\centering \footnotesize (c) MNIST
 				\end{minipage}%

 				\begin{minipage}[c]{0.33\textwidth}
                \centering
 						\includegraphics[clip, trim=0.40cm 0.5cm 0.45cm 0.3cm, width=0.95\linewidth]{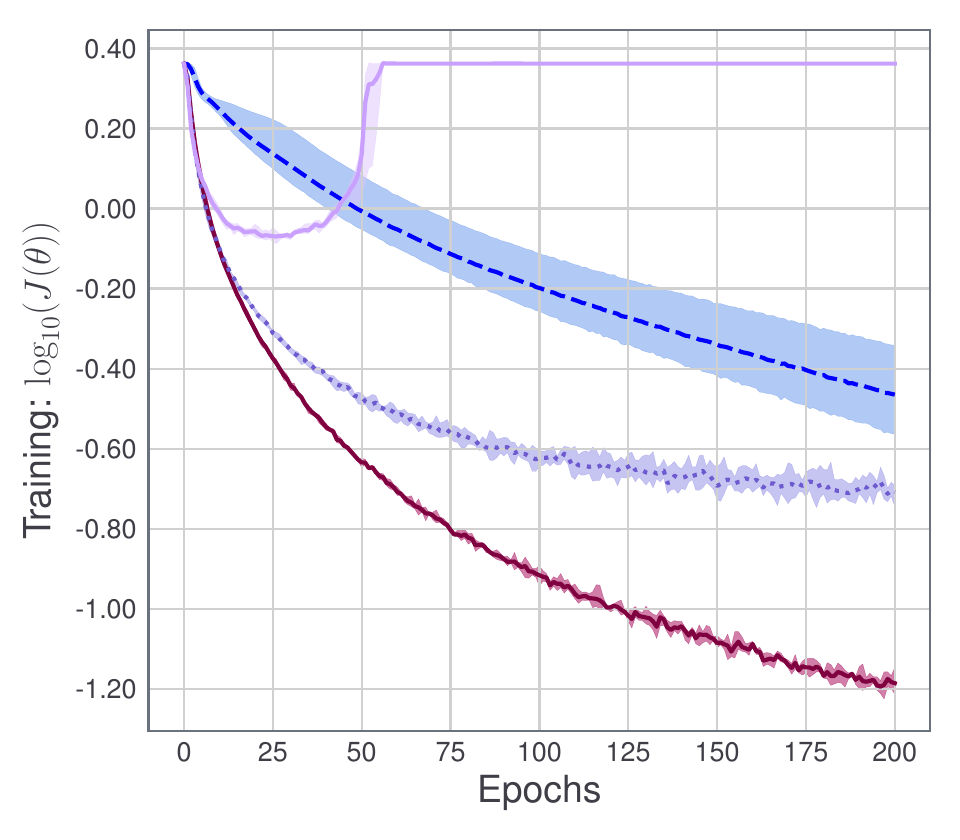}
 				\end{minipage}%
 				\begin{minipage}[c]{0.33\textwidth}
 				\centering
 						\includegraphics[clip, trim=0.40cm 0.5cm 0.35cm 0.3cm, width=0.95\linewidth]{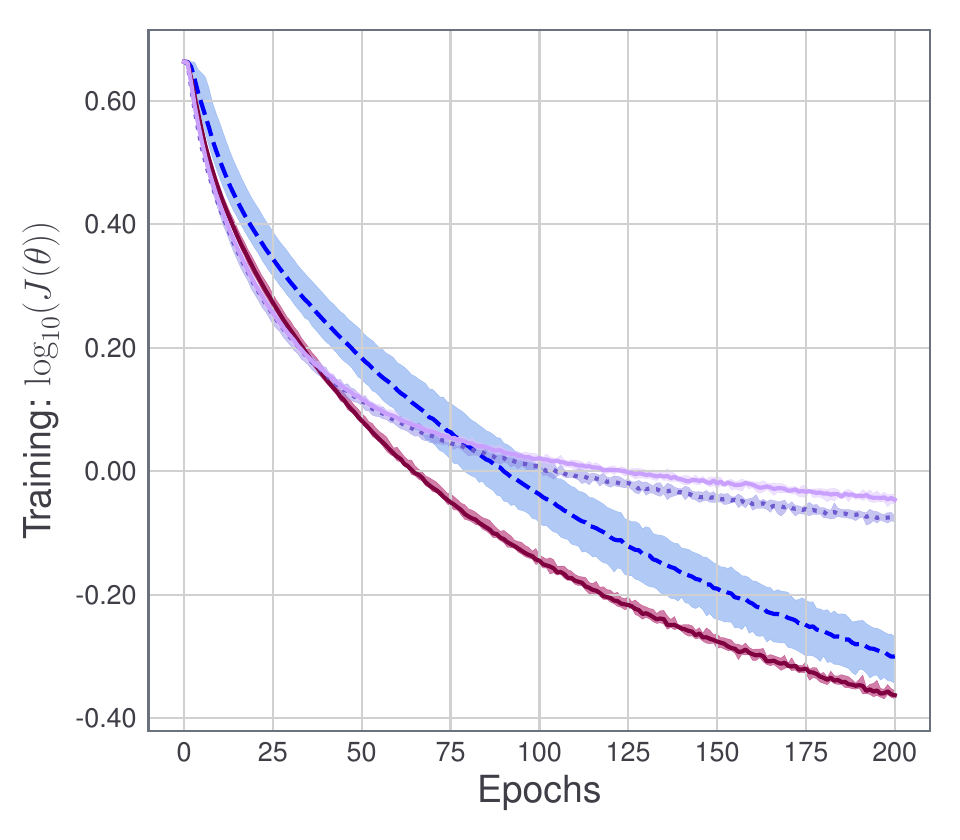}
 				\end{minipage}%
 				\begin{minipage}[c]{0.33\textwidth}
 				\centering
 						\includegraphics[clip, trim=0.40cm 0.5cm 0.35cm 0.3cm, width=0.95\linewidth]{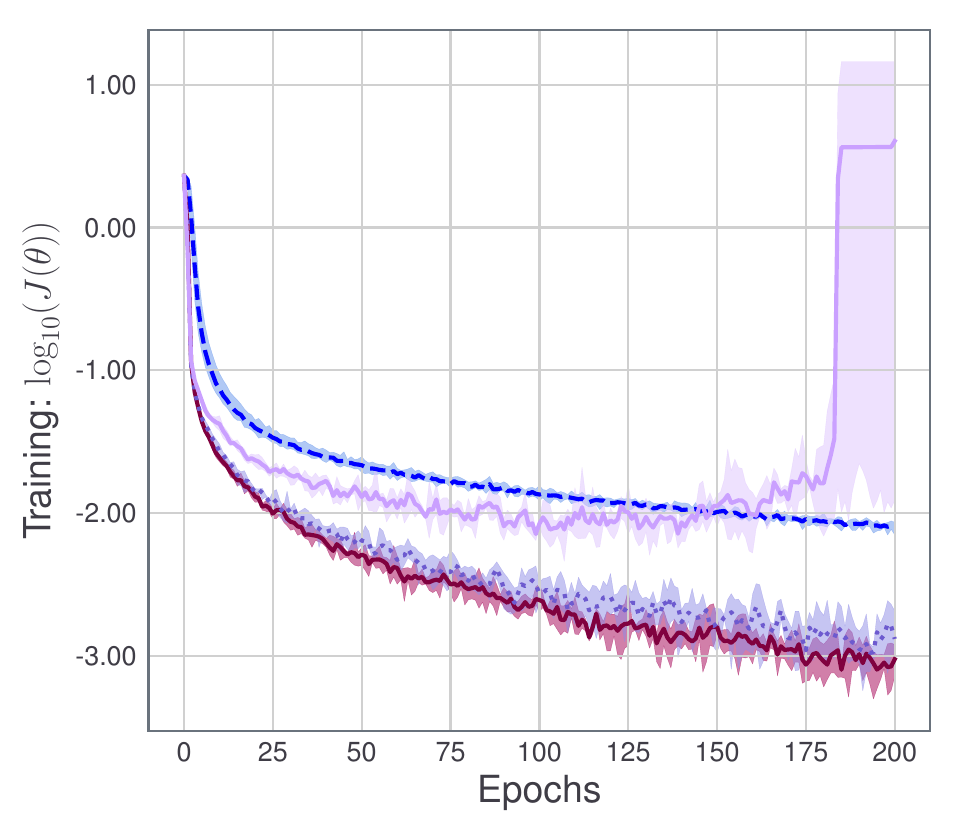}
 				\end{minipage}%

 			 	\begin{minipage}[c]{0.8\linewidth}
 			 	\centering\includegraphics[width=0.4\linewidth]{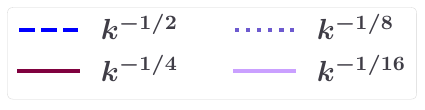}
 			 	\end{minipage}%
 			 	
 				\begin{minipage}[c]{0.33\textwidth}
                \centering
 						\includegraphics[clip, trim=0.40cm 0.5cm 0.45cm 0.3cm, width=0.95\linewidth]{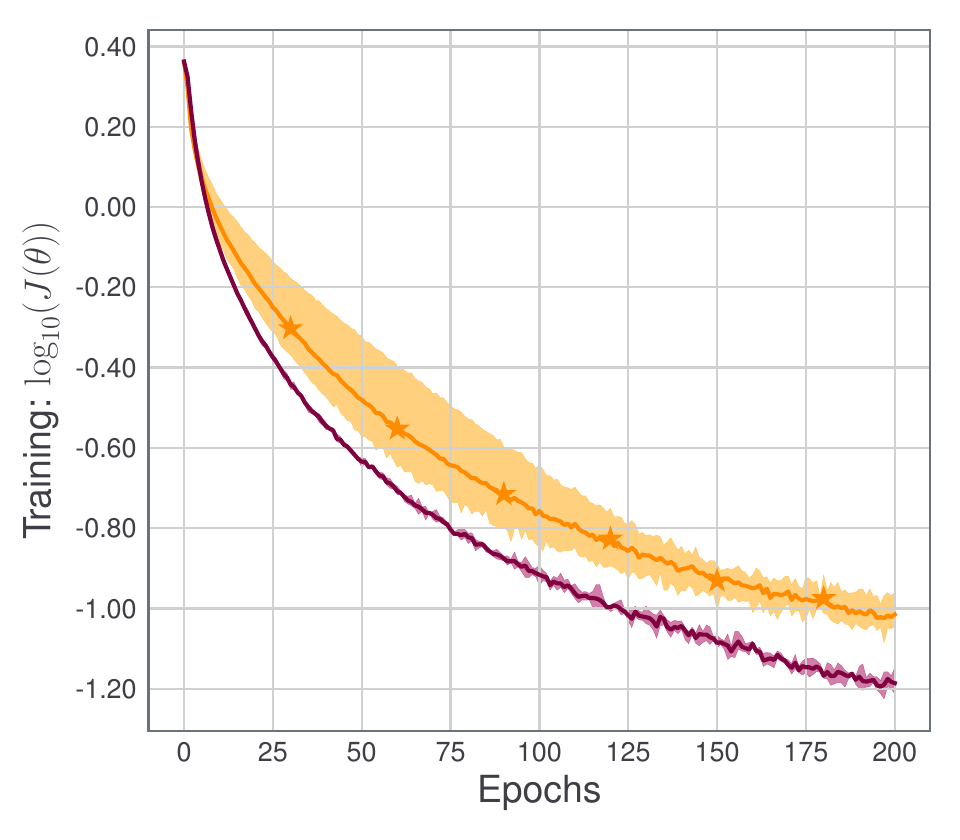}
 				\end{minipage}%
 				\begin{minipage}[c]{0.33\textwidth}
 				\centering
 						\includegraphics[clip, trim=0.40cm 0.5cm 0.35cm 0.3cm, width=0.95\linewidth]{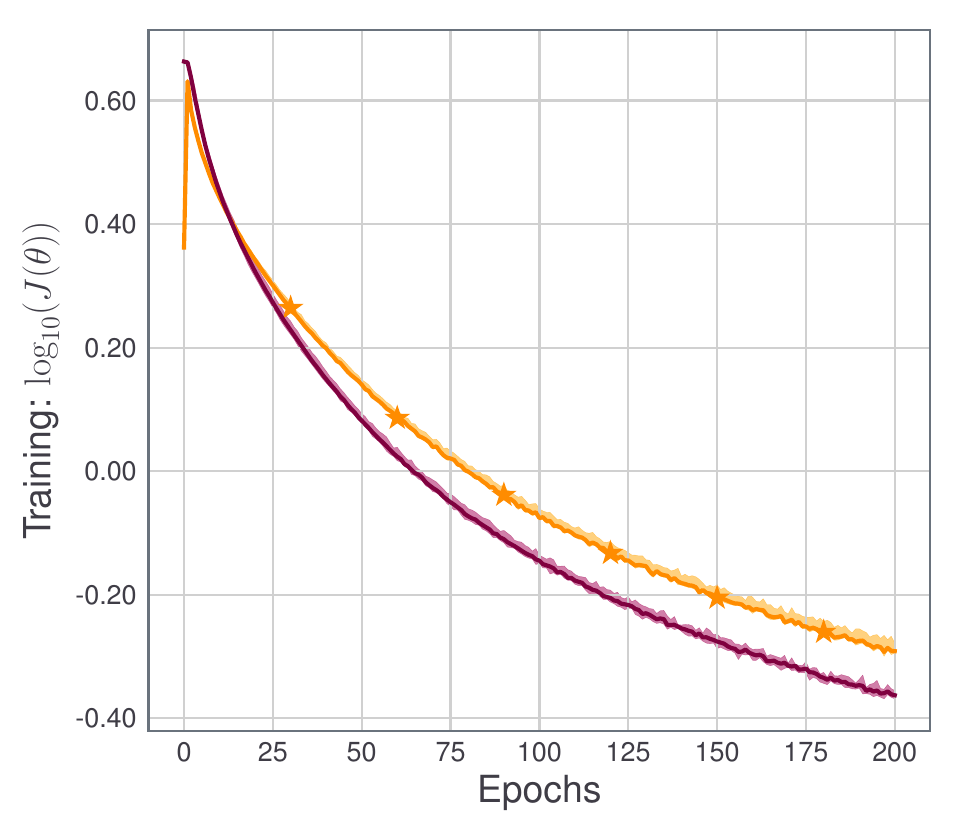}
 				\end{minipage}%
 				\begin{minipage}[c]{0.33\textwidth}
 				\centering
 						\includegraphics[clip, trim=0.40cm 0.5cm 0.35cm 0.3cm, width=0.95\linewidth]{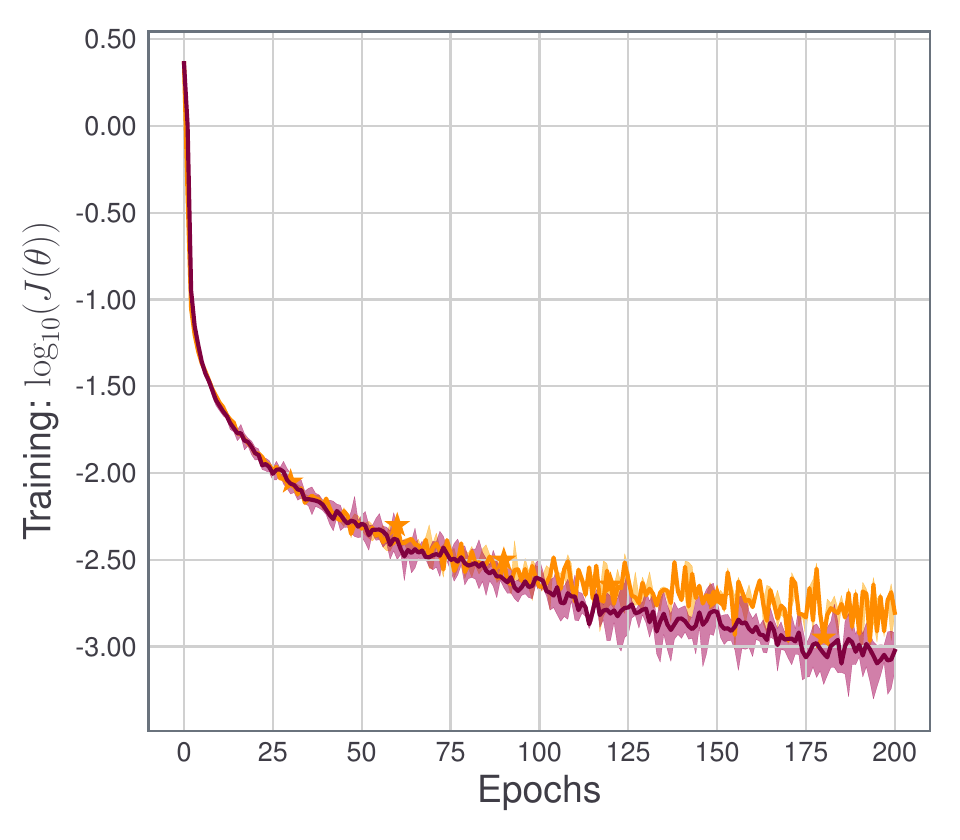}
 				\end{minipage}%

 			 	\begin{minipage}[c]{0.8\linewidth}
 			 	\centering\includegraphics[width=0.4\linewidth]{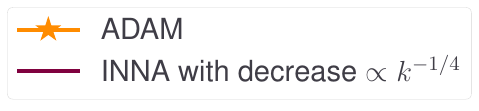}
 			 	\end{minipage}%
 			 	\caption{On top: Training loss of INNA on three image classification problems with various step-size decays. In the legend, $k^{-q}$ means a step-size decay at iteration $k$ of the form $\gamma_k = \gamma_0 k^{-q}$. The bottom row show the comparison between INNA with a well-chosen step-size decay and ADAM.}\label{fig::decay}
 			 	\end{center}
 		\end{figure}

        Finally, let us point out that although ADAM was faster in the experiments of Figure~\ref{fig::vsothalgo}, INNA can outperform ADAM using the slow step-size decay discussed in Section~\ref{sec::comRes}. Indeed, in the previous experiments we used a standard decreasing step-size of the form $\gamma_0 / \sqrt{k+1}$ for simplicity, but Assumption~\ref{ass:mainAssumption} allows for step-sizes decreasing much more slowly. As such, we also considered decays of the form $\gamma_0 (k+1)^{-q}$ with $q\leq 1/2$. The results are displayed on top of Figure~\ref{fig::decay}. 
        Except when $q$ is too small (too slow decay, e.g., $q=1/16$), these results show that some decays slower than $q=1/2$ make INNA a little faster than any of the other algorithms we tried. In particular, with a step-size decay proportional to $k^{-1/4}$, INNA outperforms ADAM (bottom of Figure~\ref{fig::decay}). This suggests that tuning $q$ can also significantly accelerate the training process.

	\section{Conclusion}
	We introduced a novel stochastic optimization algorithm featuring inertial and Newtonian behavior motivated by applications to deep learning. We provided a powerful algorithmic convergence analysis under weak hypotheses applicable to most DL problems. We also provided new general results to study differential inclusions on Clarke subdifferential and obtain convergence rates for the continuous-time counterpart of our algorithm. We would like to point out that, apart from SGD \citep{davis2018stochastic}, the convergence of concurrent methods in such a general setting is still an open question. Our result seems moreover to be the first one to be able to rigorously handle the analysis of mini-batch sub-sampling for ReLU DNNs via the introduction of the $D$-critical points. Our experiments show that INNA is very competitive with state-of-the-art algorithms for DL but also very malleable. We stress that these numerical manipulations were performed on substantial DL benchmarks with only limited algorithm tuning. 
	This facilitates reproducibility and allows staying as close as possible to the reality of DL applications in machine learning.
	
 	\section*{Acknowledgments}
     The authors acknowledge the support of the European Research Council (ERC FACTORY-CoG-6681839), the Agence Nationale de la Recherche (ANR 3IA-ANITI, ANR-17-EURE-0010 CHESS, ANR-19-CE23-0017 MASDOL) and the Air Force Office of Scientific Research (FA9550-18-1-0226).

     Part of the numerical experiments were done using the OSIRIM platform of IRIT, supported by the CNRS, the FEDER, Région Occitanie and the French government (\url{http://osirim.irit.fr/site/en}). We thank the development teams of the following libraries that were used in the experiments: Python \citep{rossum1995python}, Numpy \citep{walt2011numpy}, Matplotlib \citep{hunter2007matplotlib}, Pytorch \citep{paszke2019pytorch}, Tensorflow and Keras \citep{abadi2016tensorflow,chollet2015}.
         
    The authors thank the anonymous reviewers for their comments which helped to improve the paper and thank Hedy Attouch and Sixin Zhang for useful discussions. 

    \DeclareRobustCommand{\VAN}[3]{#3} 
	\bibliography{biblio.bib}

\end{document}